\documentclass[conference,compsoc]{IEEEtran}

\usepackage{booktabs}
\usepackage{pifont}
\newcommand{\cmark}{\ding{51}}
\newcommand{\xmark}{\ding{55}}
\usepackage{adjustbox}
\usepackage{amsmath, amssymb, amsthm}
\theoremstyle{plain}
\newtheorem{theorem}{Theorem}
\usepackage{tablefootnote}

\ifCLASSOPTIONcompsoc
  \usepackage[nocompress]{cite}
\else
  \usepackage{cite}
\fi
\ifCLASSINFOpdf
\else
\fi
\usepackage{hyperref}
\hypersetup{ 
    colorlinks=true,   
    citecolor=blue,    
    linkcolor=blue,    
    urlcolor=blue      
}
\usepackage{cleveref}
\begin{document}
%
\title{Quantifying Membership Disclosure Risk\\for Tabular Synthetic Data Using Kernel Density Estimators}



%
\author{\IEEEauthorblockN{Rajdeep Pathak\IEEEauthorrefmark{1}\IEEEauthorrefmark{2},
Amit Basak\IEEEauthorrefmark{3} and
Sayantee Jana\IEEEauthorrefmark{3}}
\IEEEauthorblockA{\IEEEauthorrefmark{1}SCAI, Sorbonne Universit\'e, Paris, France}
\IEEEauthorblockA{\IEEEauthorrefmark{2}SAFIR, Sorbonne University Abu Dhabi, UAE
}
\IEEEauthorblockA{\IEEEauthorrefmark{3}Indian Institute of Technology Hyderabad, India}}


\maketitle

\begin{abstract}
The use of synthetic data has become increasingly popular as a privacy-preserving alternative to sharing real datasets, especially in sensitive domains such as healthcare, finance, and demography. However, the privacy assurances of synthetic data are not absolute, and remain susceptible to membership inference attacks (MIAs), where adversaries aim to determine whether a specific individual was present in the dataset used to train the generator. In this work, we propose a practical and effective method to quantify membership disclosure risk in tabular synthetic datasets using kernel density estimators (KDEs). Our KDE-based approach models the distribution of nearest-neighbor distances between synthetic data and the training records, allowing probabilistic inference of membership and enabling robust evaluation via ROC curves. We propose two attack models: a `True Distribution Attack' which assumes privileged access to training data, and a more realistic, implementable `Realistic Attack' that uses auxiliary data without true membership labels. Empirical evaluations across four real-world datasets and six synthetic data generators demonstrate that our method consistently achieves higher F1 scores and sharper risk characterization than a prior baseline approach, without requiring computationally expensive shadow models. We would like to highlight that the proposed method is $200\text{K}\times$ faster on average as compared to state-of-the-art shadow model-based MIAs, while preserving attack accuracy. Our approach, therefore, provides a practical framework and metric for quantifying membership disclosure risk in synthetic data, which enables data custodians to conduct a post-generation risk assessment prior to releasing their synthetic datasets for downstream use. The datasets and codes for this study are available at \url{https://github.com/PyCoder913/MIA-KDE}.
\end{abstract}


%
\IEEEpeerreviewmaketitle

\section{Introduction} \label{sec:Introduction}
Imagine sharing rich and sensitive healthcare, financial, or demographic data without risking individual privacy: synthetic data generation makes this possible. By training statistical or machine learning models on real data to simulate datasets, synthetic data aims to retain analytical utility while protecting individual identities. However, despite growing adoption, synthetic datasets remain susceptible to membership inference attacks (MIAs) \cite{chen2020gan}, where adversaries aim to infer whether specific individuals were present in the original training data. This is particularly concerning when the inclusion of an individual in the dataset reveals sensitive information, thereby breaching privacy (e.g., HIV status, rare diseases, financial default).

The scope of membership inference attacks has recently expanded beyond traditional classifiers and early generative models, exposing critical vulnerabilities across diverse machine learning paradigms. Contemporary research has demonstrated successful MIAs against transfer learning frameworks \cite{wu2024rethinking}, sequence models \cite{rossi2025membership}, recommender systems \cite{he2025membership}, pre-trained large language models (LLMs) \cite{he2025towards}, vision-language models \cite{hu2025membership}, video learning models \cite{li2026vid}, and graph neural networks \cite{lassila2026practical}. However, the vast majority of this recent literature focuses on formulating increasingly potent attacks against `models' from the perspective of an external adversary aiming to breach privacy. While adversarial exploits effectively highlight model vulnerabilities, data custodians tasked with sharing synthetic data (generated using any classical or deep generative method) require actionable, post-generation risk scoring metrics. Such defensive tools are essential to systematically quantify membership disclosure risks in synthetic datasets prior to their public release. 

Recent MIA strategies against synthetic data can be generally classified into three major categories: shadow model-based, density estimation-based, and distance-based. However, they often force a trade-off between computational efficiency and predictive utility. State-of-the-art shadow modeling \cite{shokri2017MIA_ML_2, stadler2022synthetic, guepin2023synthetic} yields robust membership predictions, but incurs severe computational overhead due to the requisite training of multiple shadow generators. Alternatively, distance-based methods \cite{el2022MIA_synth, mendelevitch2021fidelity, choi2017generating} drastically reduce this overhead through nearest-neighbor heuristics, yet output strict binary membership labels and utilize average-case metrics like F1 score, that prohibit uncertainty quantification and downstream ROC evaluations. ROC analysis is important to assess the true positive membership classification rate at low false positive rate (FPR) regimes, since average-case metrics under-represent attack risks \cite{carlini2022MIA_ML_3}. Consequently, there is a need for an optimal methodology to quantify MIA risk for tabular synthetic datasets, that bridges this gap by delivering the probabilistic granularity and attack strength of shadow models with the computational tractability of distance-based techniques.

To this end, we introduce a non-parametric framework using Kernel Density Estimators (KDEs) \cite{parzen1962estimation} to independently model the empirical distance distributions of members and non-members relative to synthetic data. For a given query record, its nearest-neighbor distance in the synthetic set is computed, and Bayes' theorem is applied to the fitted densities to calculate the posterior probability of training set membership, $\mathbf{P}(\text{member}\vert{}d)$. Our approach generates probabilistic membership predictions, enabling improved analysis through ROC curves at low FPRs while remaining computationally efficient. The proposed method differs from each of the existing approaches in several important ways. First, unlike shadow modeling attacks, our framework requires no surrogate generator training, making it considerably easier to deploy while remaining generator-independent. Second, unlike density estimation-based methods like DOMIAS \cite{van2023membership}, which estimates densities in the original feature space, the proposed method models the one-dimensional distribution of nearest-neighbor distances, substantially simplifying density estimation while preserving informative signals for membership inference. Finally, the proposed approach replaces the threshold-based decision in distance-based methods with a probabilistic formulation, allowing robust risk evaluation via ROC analysis.

\begin{table}[!t]
\centering
\caption{Qualitative comparison of representative membership inference attacks against synthetic tabular data.}
\label{tab:comparison}
\begin{adjustbox}{width=\linewidth}
\begin{tabular}{lccccc}

\toprule
Method &
Shadow &
Generator &
Probabilistic &
ROC &
Computation \\

& Models & Retraining & Output & Analysis & \\
\midrule
Shadow model-based \cite{shokri2017membership} & \cmark & \cmark & \cmark & \cmark & High \\
Density-based (DOMIAS; \cite{van2023membership}) & \xmark & \xmark & \cmark & \cmark & Moderate \\
Distance-based \cite{el2022MIA_synth} & \xmark & \xmark & \xmark & \xmark & Low \\
\textbf{Proposed Framework} & \xmark & \xmark & \cmark & \cmark & Low \\
\bottomrule

\end{tabular}
\end{adjustbox}
\end{table}

We propose two new attack variants: a \textit{True Distribution Attack} that leverages the real training data to explicitly model nearest-neighbor distances, and a \textit{Realistic Attack} executable by an adversary possessing limited auxiliary data. We empirically validate the robustness of our approach against the standard distance-based attack \cite{el2022MIA_synth} across four real-world datasets (MIMIC-IV, UK Census, Texas100X, and Nexoid COVID-19) and six synthetic data generators spanning a variety of deep generative paradigms -- GAN-based CTGAN, ADS-GAN, and DPGAN, diffusion-based TabDDPM, VAE-based TVAE, and directed acyclic graph (DAG)-based Bayesian Network. Moreover, we perform rigorous computational comparisons, demonstrating that our framework requires significantly less overhead than state-of-the-art shadow modeling attacks, rendering it highly practical for operational deployment. In summary, our main contributions are: (1) proposing a new KDE-based privacy scoring framework for tabular synthetic data to accurately quantify membership disclosure risk; (2) providing an extensive empirical evaluation of our methodology across diverse datasets and generation models; and (3) achieving superior attack efficacy (via F1 scores) compared to traditional data-partitioning methods, without the computational expense of shadow model training.

\section{Background and Related Work} \label{sec:related_work}

\subsection{Synthetic Data Generation}
Synthetic data is generated by first training a statistical or machine learning model on the original dataset, and then producing artificial (synthetic) records by sampling new values from the learned model. Contemporary literature highlights diverse methodologies for synthetic data generation. Bayesian networks (BayNet, PrivBayes \cite{zhang2017privbayes}) model attribute dependencies via DAGs, while Generative Adversarial Networks (GANs) (e.g., CTGAN \cite{xu2019modeling}) utilize adversarial training despite inherent challenges with tabular data. Generative alternatives like Variational Autoencoders (VAEs) \cite{vardhan2020generating, ma2020vaem} and diffusion models \cite{kotelnikov2023tabddpm} have also become prominent. Most recently, LLMs have demonstrated strong capabilities in capturing real-world tabular distributions. However, they face practical limitations; frameworks like GReaT \cite{borisov2023language} suffer from slower convergence than CTGAN, while methods such as TAPTAP \cite{zhang2023generative}, TabuLa \cite{zhao2025tabula}, and TabMT \cite{gulati2023tabmt} necessitate dedicated fine-tuning to effectively model target distributions.


\subsection{Membership Inference Attacks}

\begin{figure*}[t]
    \centering
    \includegraphics[width=0.8\textwidth]{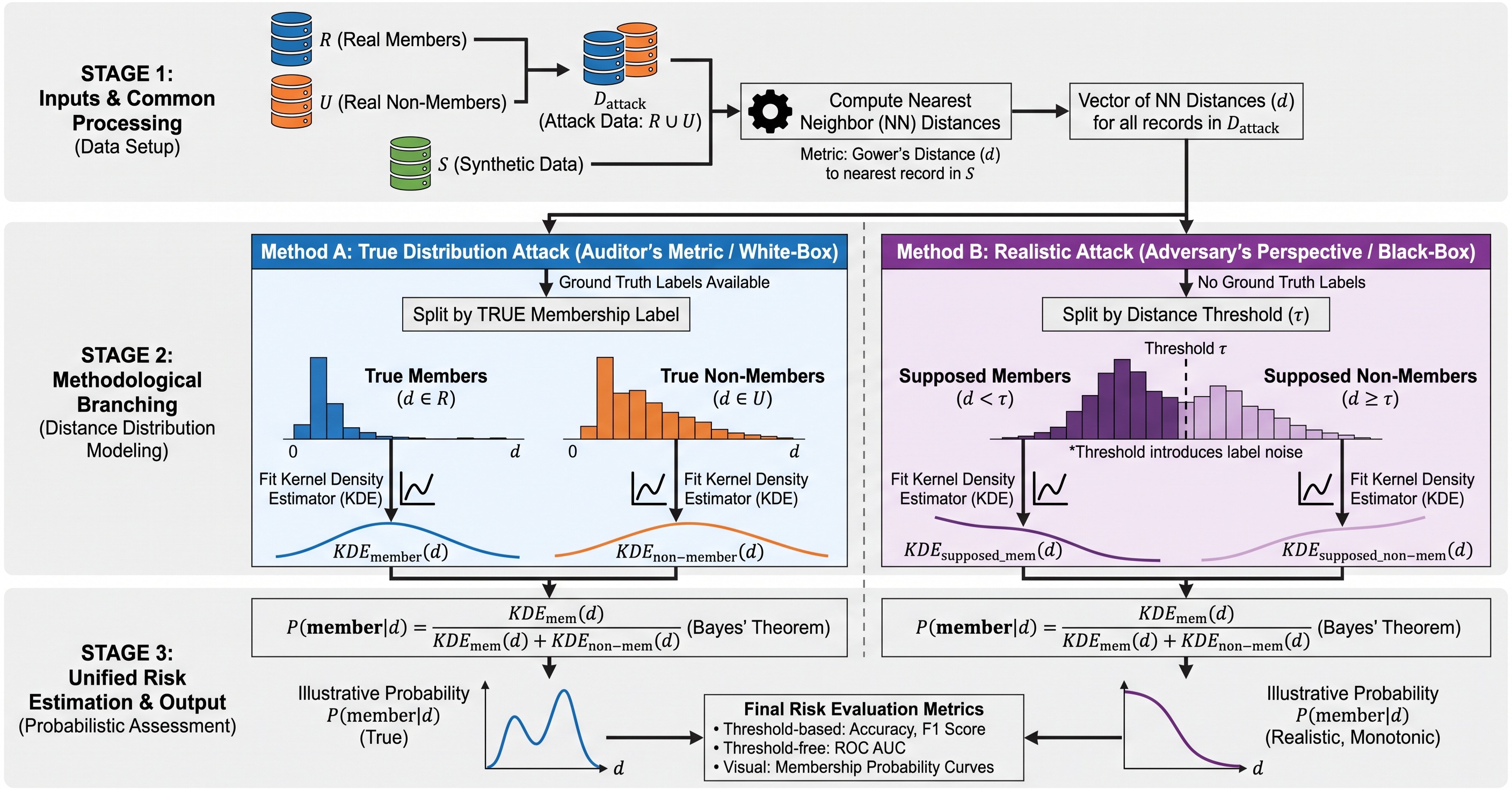}
    \caption{Overview of the proposed methodology: True Distribution Attack and Realistic Attack for quantifying membership disclosure risk for tabular synthetic data.}
    \label{fig:mia_method}
\end{figure*}
Membership inference attacks (MIAs) have emerged as a standard approach for assessing the privacy risks associated with synthetic data and machine learning models. For synthetic tabular data, one class of MIAs involves direct comparisons between synthetic and original records to identify exact or near matches \cite{yale2019assessing}. However, these similarity-based approaches have been criticized for underestimating privacy risks \cite{stadler2022synthetic}. 

The state-of-the-art alternatives are shadow modeling approaches \cite{shokri2017MIA_ML_2, stadler2022synthetic, guepin2023synthetic}, which are practically infeasible for large datasets due to their heavy computational overhead. For instance, consider a data custodian who routinely generates and distributes synthetic datasets derived from proprietary, high-dimensional, and dynamic data containing billions of records, refreshed weekly. To assess privacy risk before each release, employing shadow modeling would require training numerous additional generative models -- an approach that is both time and resource-intensive. Given the dynamic nature of the data, any delay in release diminishes its utility and commercial value to potential users. 

A more popular approach implemented in practice to evaluate membership disclosure risks of synthetic data is a data-partitioning and distance-based approach \cite{mendelevitch2021fidelity, el2022MIA_synth, choi2017generating, chen2020gan, zhang2020ensuring, goncalves2020generation, yan2021generating}, which we elaborate in Section \ref{sec:methods}. It utilizes the nearest-neighbor distance of a record from a synthetic data to determine whether the individual was a part of the real training data. Our proposed method is grounded on the distance-based method, where we model the distance distributions using KDEs.

\section{Methodology} \label{sec:methods}


Membership inference attacks simulate a classification task where adversaries classify records as training set members or non-members. Given real dataset $R$ (training set) generating synthetic data $S$, and unseen data $U$ (not used in training), we form an attack dataset $D_{\text{attack}}$ comprising records from $R$ and $U$. The standard distance-based approach in \cite{mendelevitch2021fidelity, el2022MIA_synth, choi2017generating, chen2020gan}, which we hereafter refer to as \textbf{Method 1}, sets a distance threshold $\tau$, computes nearest-neighbor distances between records in $D_{\text{attack}}$ and synthetic data $S$, and classifies records below the threshold as members, and others as non-members. Finally, the F1 score for this classification quantifies the membership disclosure risk. The intuition behind this method is that, records in the attack (auxiliary) dataset closer to the synthetic dataset are more likely to be members of the real training data.

However, \cite{carlini2022MIA_ML_3} argue that average-case metrics like F1 scores underestimate privacy risks by masking worst-case leakage scenarios. They recommend ROC analysis evaluating true positive rates (TPRs) at low false positive rates (FPRs). Method 1 produces only hard classifications, and not probabilistic membership scores needed for comprehensive ROC analysis. In many settings, it is beneficial to have a probability associated with the class label. Membership probabilities provide a sense of confidence, the degree of uncertainty, or the likelihood of being incorrect in classifying records as members.

We introduce a Kernel Density Estimation (KDE)-based framework that builds upon the structural foundations of Method 1 to yield probabilistic membership predictions. Specifically, our approach models the empirical relationship between nearest-neighbor distances and membership likelihoods utilizing the real training records, the synthetic dataset, and an unseen holdout set. This formulation is inherently generator-agnostic; it can be applied to synthetic data produced by any generative paradigm without necessitating the computationally prohibitive retraining of auxiliary models, thereby offering a highly efficient alternative to resource-intensive shadow modeling techniques. We present two attack variants: the \textbf{True Distribution Attack}, leveraging data holders' privileged access to training data for risk assessment, and the \textbf{Realistic Attack}, using auxiliary datasets accessible to adversaries, which was also a requirement in existing literature \cite{el2022MIA_synth, shokri2017MIA_ML_2}. Note that, the interpretation of membership inference risk is two-fold: low risk may indicate genuine privacy protection, or at the same time, reflect weak attack models that underestimate actual leakage. 

\subsection{True Distribution Attack} \label{sec:Upper_bound_attack}
Let $R$ denote the real training dataset, $S$ be the synthetic dataset, and $U$ be the unseen real records not used during training. We construct the attack dataset $D_{\text{attack}}$ by combining $R$ and $U$, with corresponding ground truth labels: $R$ records as members (class = 1) and $U$ records as non-members (class = 0). We randomly split $D_{\text{attack}}$ into balanced training and test sets -- each set has an equal number of members $(\in R)$ and non-members $(\in U)$. To facilitate an interpretable risk comparison, this approach utilizes a balanced dataset where a random-guessing adversary would achieve a 50\% baseline accuracy and a 0.67 F1 score. This establishes a foundational benchmark for evaluating the strength of our proposed attack.

We compute the Gower's distance\footnote{Absolute difference (Manhattan distance) for numerical features; 0 for matching categories, 1 otherwise.} \cite{Gower1971} between each record in $D_{\text{attack}}$ and its nearest neighbor in $S$, creating a distance matrix $D_{\text{dists}}$ with computed distances and corresponding membership labels. We fit separate KDEs to member distances ($KDE_{\text{member}}$ for distances corresponding to records originally in $R$) and non-member distances ($KDE_{\text{non-member}}$ for distances corresponding to records originally in $U$), providing smooth approximations of empirical distance distributions for each class. For evaluation, we classify test records by computing their nearest-neighbor distance $d$ to the synthetic dataset and estimating membership probability using:
\begin{equation} \label{eq:membership_probability}
    \mathbf{P}(\text{member}|d) = \frac{KDE_{\text{member}}(d)}{KDE_{\text{member}}(d) + KDE_{\text{non-member}}(d)}.
\end{equation}

This estimator, formally justified by Bayes' Theorem in \Cref{proposition:prob_justification}, models membership likelihood without applying thresholds on the distance values. Finally, for obtaining a membership label, this membership probability is thresholded (e.g., $\mathbf{P}(\text{member}|d) \geq 0.5$) to classify whether the test record is a member of the real training dataset. The probabilistic output quantifies membership uncertainty corresponding to each test query, as well as enables robust risk evaluation using both average-case classification metrics (accuracy, F1 score) and comprehensive ROC analysis. A visual illustration of the method is provided in Figure \ref{fig:mia_method}.

\begin{theorem} \label{proposition:prob_justification}
    Following the above setup, assume that $KDE_{\text{member}}$ and $KDE_{\text{non-member}}$ approximates the distributions of member distances and non-member distances, respectively. Then for a given query record and the distance $d$ of its nearest neighbor in the synthetic data, the probability that it is a member of the training set is given by $$\mathbf{P}(\text{member}|d) = \frac{KDE_{\text{member}}(d)}{KDE_{\text{member}}(d) + KDE_{\text{non-member}}(d)}.$$
\end{theorem}

The proof is provided in supplementary Section \ref{appendix:proof}. The true distribution attack requires privileged access to the real training data $R$, and serves as our main proposal for an \textit{Auditor's Metric}. Data custodians can utilize this method as a white-box privacy validator and a post-generation evaluation metric. 

\subsection{Realistic Attack} \label{sec:Realistic_attack}
In practical scenarios, adversaries will not have access to the true membership labels, making direct modeling of member and non-member distance distributions impossible. We propose the \textbf{Realistic Attack}, where adversaries access only auxiliary datasets from the same population, potentially containing mixed training and non-training records acquired through public sources. The requirement of such an auxiliary dataset is also assumed in prior works \cite{shokri2017MIA_ML_2, el2022MIA_synth} and it does not affect the validity of our approach. Our objective is providing data custodians with practical risk assessment methods, and not creating new, potentially stronger, MIAs per se. 

We propose the realistic attack using $D_{\text{attack}}$ (mixture of $R$ and $U$ without the true membership labels) as auxiliary data. Without true labels, we partition the records into \textbf{`supposed members'} (with distance below a threshold) and \textbf{`supposed non-members'} (above the threshold), based solely on nearest-neighbor distances to synthetic data. Unlike Method 1, these labels are not directly used for classification. Instead, we fit separate KDEs to the supposed member and supposed non-member distances, and evaluate membership probability of a query record using \eqref{eq:membership_probability}, enabling probabilistic assessment under realistic assumptions.

Consequently, this formulation yields membership probabilities that decrease monotonically as a function of nearest-neighbor distance. Because the adversary lacks access to ground-truth labels, the initial data partitioning inherently introduces label noise; the sets of supposed members and non-members contain unavoidable cross-class contamination, leading to minor modeling inaccuracies within the density estimators. Despite this noise, empirical evaluations demonstrate that at elevated distance thresholds (exceeding the 50th percentile), our Realistic Attack achieves significantly higher F1 scores than the established distance-based baseline (Method 1) \cite{el2022MIA_synth}. By demonstrating superior attack efficacy relative to existing methodologies, our framework furnishes data custodians with a more rigorous metric for accurately quantifying membership disclosure risk.


\section{Experimental Results} \label{sec:results}
We use four publicly available datasets to demonstrate our method: MIMIC-IV \cite{johnson2020mimic, johnson2023mimic}, UK Census \cite{ONS2011Census}, Texas-100X \cite{texas2006inpatient}, and Nexoid COVID-19 \cite{nexoid2020covid19}. These are popular healthcare and demographic datasets, extensively used in contemporary literature \cite{guepin2023synthetic, pamisetty2024ehrgpt}. We partition each of the datasets into two equal parts to construct the attack dataset: one half ($R$) is used to train the synthetic data generators and generate $S$, while the remaining half ($U$), unseen during training, is used as the unseen (non-members) set. 

We use six different generation techniques for comprehensive evaluation: CTGAN \cite{xu2019modeling}, ADS-GAN \cite{yoon2020anonymization}, DPGAN \cite{xie2018differentially}, TabDDPM \cite{kotelnikov2023tabddpm}, TVAE \cite{xu2019modeling}, and Bayesian Network \cite{young2009using}, which are trained on the real data $R$. The generators are implemented using the SynthCity framework \cite{synthcity}. More details about the datasets and generation mechanisms are provided in the supplementary Sections \ref{appendix:datasets} and \ref{appendix:generators}, respectively.

We evaluate both attack variants across these 24 datasets using Gower's distance between $D_{\text{attack}}$ and $S$, measuring MIA risk via attack accuracy, F1 score, and log-scaled ROC curves. First, we form $D_{\text{dists}}$ by computing the distances between records in $D_{\text{attack}}$ and $S$. We randomly split $D_{\text{dists}}$ into 70\% training (KDE fitting) and 30\% testing (with balanced members and non-members), using $0.5$ membership probability threshold for classification and Scott's rule for KDE bandwidths. The specific choice of bandwidth matters in practice, however, the analysis and results hold for any choice. For true distribution attack, we use an equal number of members and non-members in the training subset. Table \ref{table:training_test_distances} shows training and test set sizes across the four datasets.
\begin{table}[!ht]
    \centering
    \caption{Number of training and test distances for each dataset.}
    \begin{adjustbox}{width=\linewidth}
    \begin{tabular}{ccccc}
    \toprule
        ~ & MIMIC-IV & UK Census & Texas-100X & Nexoid \\ \midrule
        \#Training distances & 112,000 & 398,818 & 647,588 & 420,000 \\ 
        \#Test distances & 48,000 & 170,923 & 277,540 & 180,000 \\ \bottomrule
    \end{tabular}
    \end{adjustbox}
    \label{table:training_test_distances}
\end{table}

Table~\ref{table:UB_RawDist_Combined} reports the accuracies and F1 scores for the true distribution attack across the various datasets and generative models. In majority of the cases, the synthetic datasets generated using Bayesian Network are observed to be comparatively more vulnerable to membership inference attacks, consistently exhibiting higher accuracy and F1 scores. Figure \ref{fig:MIA_MCB} ranks the synthetic data generators by their susceptibility to MIAs using the Multiple Comparisons with the Best (MCB) test \cite{koning2005m3}, where lower MIA accuracy denotes greater structural robustness. DPGAN achieves the best rank, successfully leveraging its differential privacy guarantees to withstand MIAs better. The Bayesian Network, lacking such protections, proves the most vulnerable and ranks last. Following \cite{carlini2022MIA_ML_3}, we also analyze ROC curves to assess worst-case privacy leakage, by evaluating TPR at low FPR regimes (see Section \ref{discussion:ROC}).

\begin{table}[!ht]
    \centering
    \caption{Membership inference accuracies and F1 scores for the true distribution attack on various datasets and generators.}
    \begin{adjustbox}{width=\linewidth}
    \begin{tabular}{llcccccc}
    \toprule
        Metric & Dataset & CTGAN & ADS-GAN & DPGAN & TabDDPM & TVAE & Bayesian Network \\ \midrule
        Accuracy & MIMIC-IV & 75.617\% & 73.708\% & 59.433\% & 79.508\% & \textbf{88.952\%} & 68.298\% \\ 
        ~ & UK Census & 50.075\% & 49.918\% & 57.933\% & 50.055\% & 49.975\% & \textbf{62.889\%} \\ 
        ~ & Texas-100X & 83.957\% & 71.184\% & 71.023\% & 50.49\% & 78.341\% & \textbf{97.503\%} \\ 
        ~ & Nexoid & 50.827\% & 50.353\% & 50.057\% & 51.6\% & 53.768\% & \textbf{58.493\%} \\ \midrule
        F1 Score & MIMIC-IV & 0.793 & 0.791 & 0.612 & 0.827 & \textbf{0.877} & 0.759 \\ 
        ~ & UK Census & 0.333 & 0.402 & 0.274 & \textbf{0.633} & 0.375 & 0.631 \\ 
        ~ & Texas-100X & 0.825 & 0.760 & 0.716 & 0.454 & 0.781 & \textbf{0.975} \\ 
        ~ & Nexoid & 0.516 & 0.079 & 0.042 & 0.579 & 0.540 & \textbf{0.619} \\ \bottomrule
    \end{tabular}
    \end{adjustbox}
    \label{table:UB_RawDist_Combined}
\end{table}

\begin{figure}[h!] 
    \centering
    \includegraphics[width=0.75\linewidth]{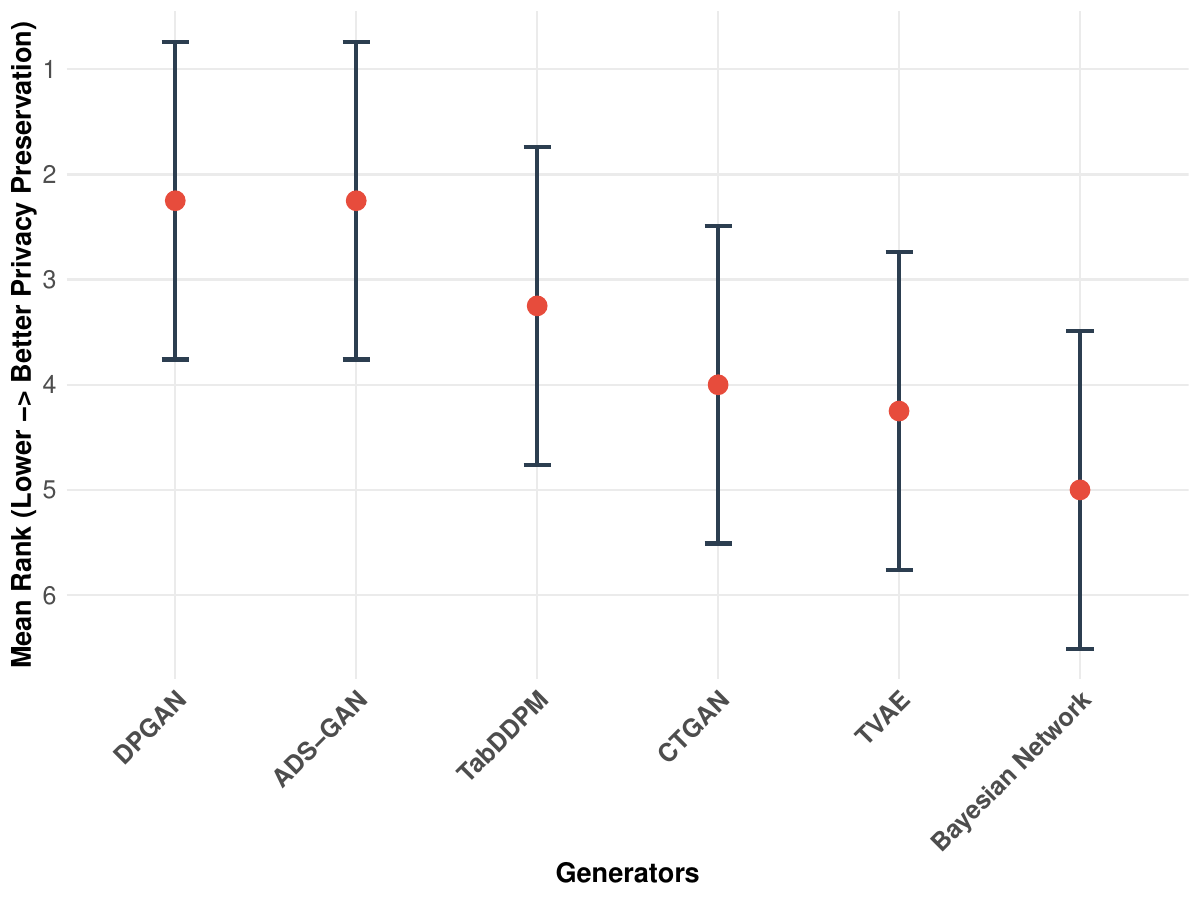}
    \caption{MCB test rankings evaluating generator robustness against MIAs. Lower mean ranks correspond to lower MIA accuracy, denoting superior privacy preservation.}
    \label{fig:MIA_MCB}
    \end{figure}

For the realistic attack, we conducted our experiments using nine distinct distance thresholds, corresponding to the $10^{th}, \ 20^{th}, \ldots, 90^{th}$ percentiles of the nearest neighbor distances. For each threshold, the distances are partitioned as `supposed member' and `supposed non-member' distances, and separate KDEs are fitted to model their distributions. We report the resulting F1 scores achieved by the realistic attack for each threshold, and compare these with the corresponding scores obtained using Method 1 through heatmaps. 
Experimental details such as nearest-neighbor distance computations and KDE fitting are deferred to supplementary Section \ref{appendix:experiments}. The following subsections present the results for both attack variants on the four datasets.


\subsection{True Distribution Attack}
Results for the true distribution attack are presented in Table \ref{table:UB_RawDist_Combined}, which serves as a white-box privacy validator for the data custodian. When member and non-member distance distributions are statistically indistinguishable, the true distribution attack yields poor performance below baseline levels. For example, TVAE-generated UK Census data shows statistically indistinguishable distances (KS test $p$-value = 0.71), resulting in accuracy (49.97\%) and F1 score (0.375) below baselines. While this suggests low average MIA risk, which might be a good news for the data custodian, ROC analysis (Figure \ref{fig:UKCensus_UB_RawDist_ROC}) reveals high TPRs at low FPRs, indicating successful attacks in worst-case scenarios (discussed in Section \ref{discussion:ROC}).

For the MIMIC-IV synthetic datasets, KS tests ($p$-values = 0) confirm statistically distinguishable member and non-member distance distributions. TVAE-generated MIMIC-IV synthetic data shows the highest vulnerability with $88.952\%$ accuracy and $0.877$ F1 score. Figure \ref{fig:EHR_UB_KDE} shows member/non-member distance distributions, with fitted KDEs used to compute test record membership probabilities. 


\begin{figure}[h!]
    \centering
    \includegraphics[width=\linewidth]{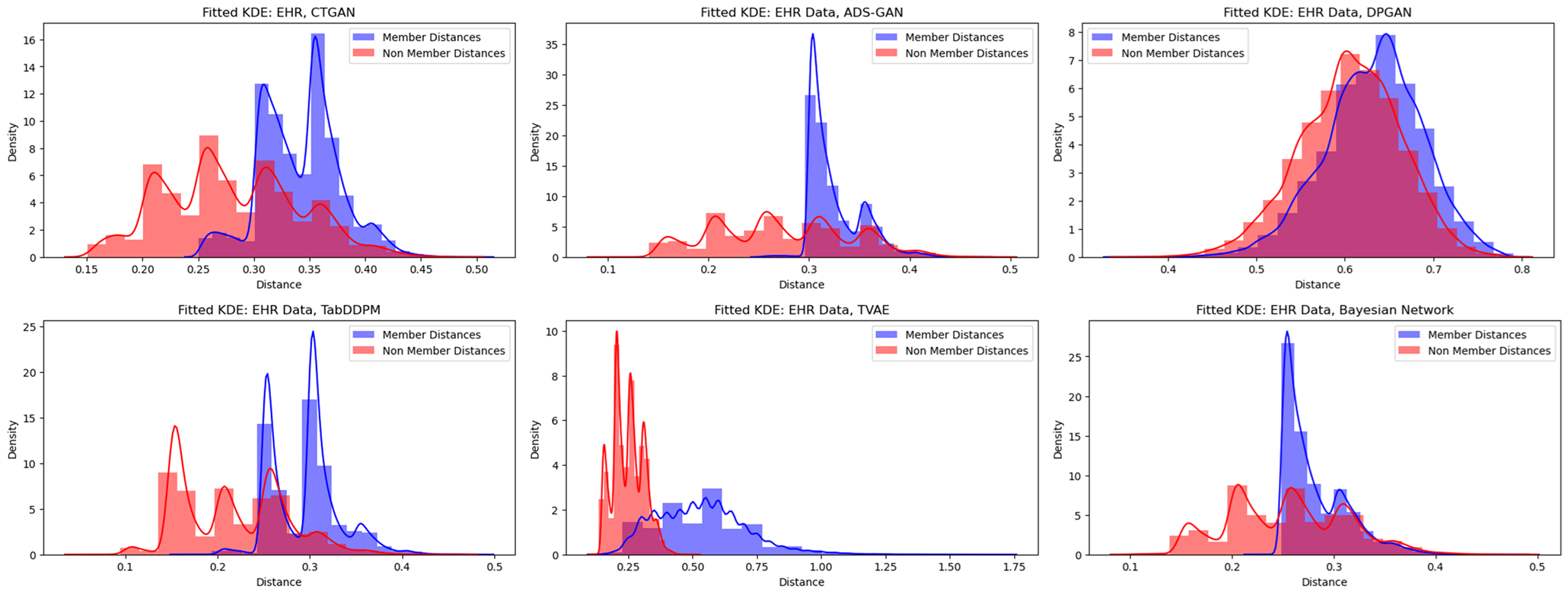}
    \caption{MIMIC-IV data: distribution of nearest-neighbor distances between the training records (in $D_{\text{attack}}$) and various synthetic datasets.}
    \label{fig:EHR_UB_KDE}
    \end{figure}
    


The UK Census dataset contains only categorical features, resulting in 7-14 discrete distance values. KS tests show indistinguishable member and non-member distances for CTGAN-, ADS-GAN-, TabDDPM-, and TVAE-generated synthetic datasets, with true distribution attack accuracies near the 50\% baseline, indicating reduced MIA susceptibility from a data holder's perspective. Figure \ref{fig:UKCensus_UB_RawDist_ROC} presents log-scaled ROC curves for true distribution attacks on UK Census synthetic datasets, demonstrating our method's capability for comprehensive ROC analysis through probabilistic predictions.

\begin{figure}[h!]
    \centering
    \includegraphics[width=0.8\linewidth]{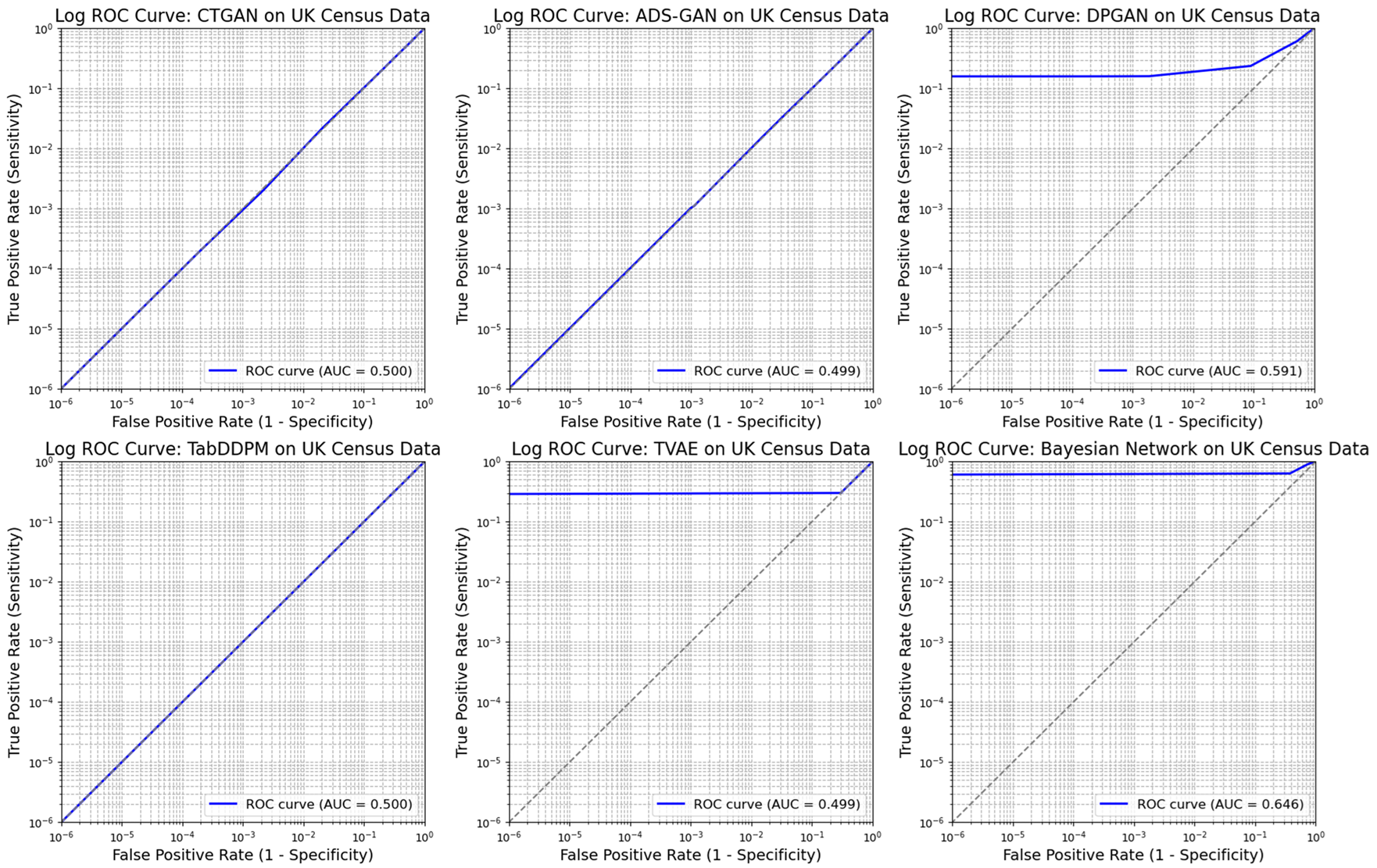}
    \caption{UK Census data: Log-ROC curves for the true distribution attack.}
    \label{fig:UKCensus_UB_RawDist_ROC}
    \end{figure}


Figure \ref{fig:Texas_UB_DistvsProb} illustrates the distance versus membership probability for the Texas-100X dataset. Clearly, the relationship shows a monotonically decreasing trend for all synthetic datasets, indicating that the probability of a record being a member of the real data decreases if it is at a greater distance from the synthetic data. For brevity, the ROC curves for the true distribution attack on Texas-100X and Nexoid datasets are deferred to the supplementary material (Figures \ref{fig:Texas_UB_RawDist_ROC} and \ref{fig:Nexoid_UB_RawDist_ROC}, respectively).

\begin{figure}[t!]
    \centering
    \includegraphics[width=0.8\linewidth]{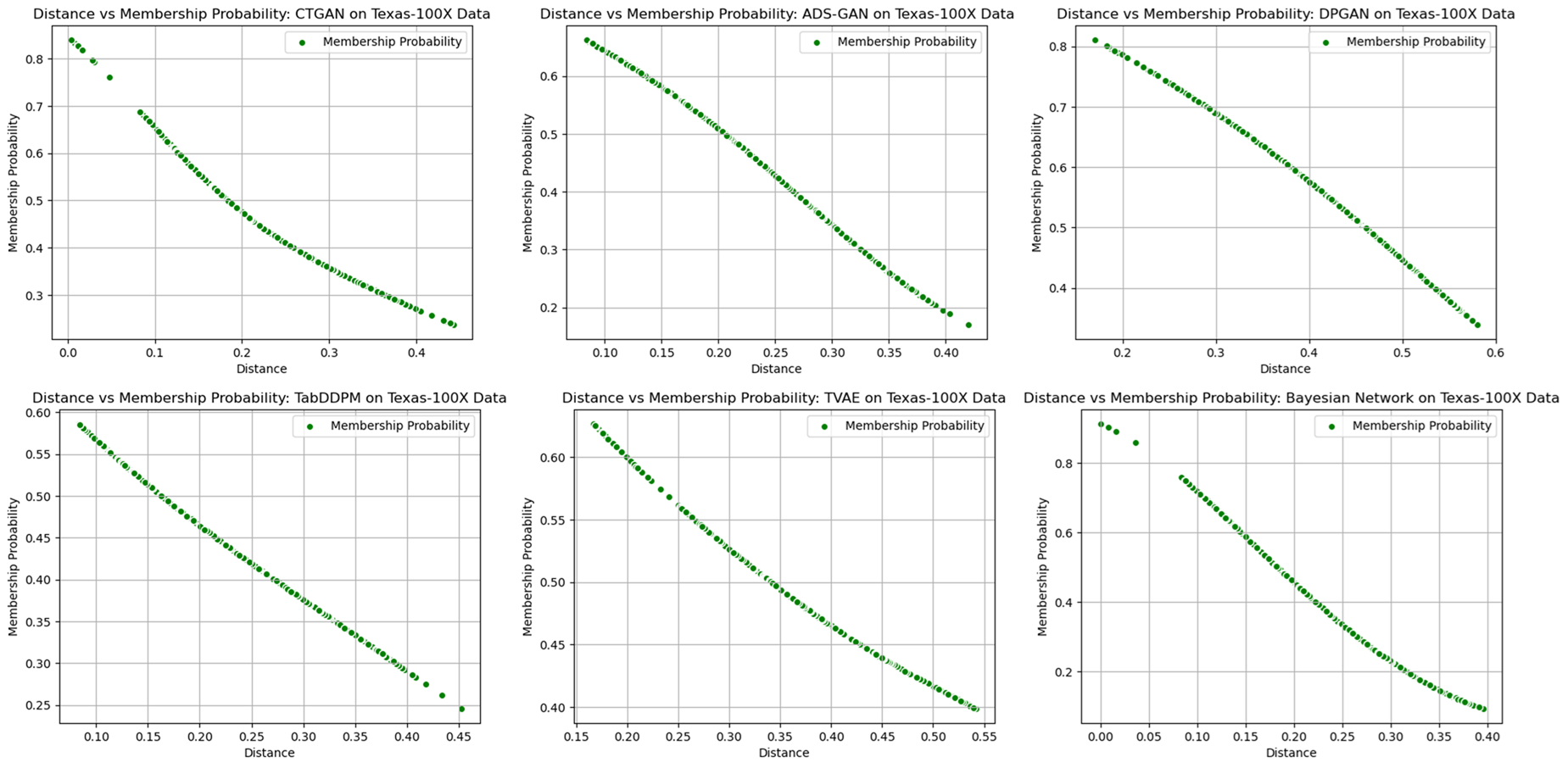}
    \caption{Texas-100X data: distance vs. membership probability for the test records.}
    \label{fig:Texas_UB_DistvsProb}
    \end{figure}

\subsection{Realistic Attack}

Figure \ref{fig:Realistic_F1_Scores} compares the F1 scores (MIA risk) across 9 different distance thresholds, obtained by the realistic attack and Method 1 for MIMIC-IV and UK Census datasets. The realistic attack yields low F1 scores at lower thresholds for MIMIC-IV, indicating limited separability between member and non-member records in these regions. However, at higher thresholds, the F1 scores increase and approach the baseline value of 0.67. For both MIMIC-IV and UK Census datasets, the proposed method consistently outperforms the baseline Method 1 \cite{el2022MIA_synth} in terms of attack strength under the same conditions, especially at higher distance thresholds.

\begin{figure}[t!]
    \centering
    \includegraphics[width=\linewidth]{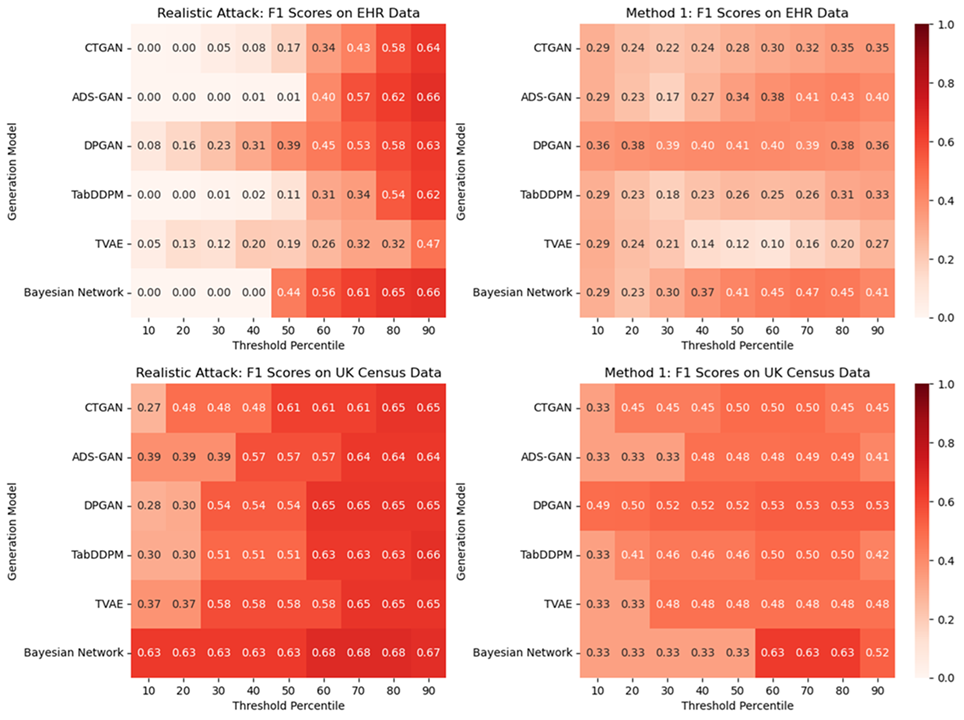}
    \caption{Comparison of F1 scores for the proposed realistic attack against Method 1 across various distance thresholds on MIMIC-IV (top) and UK Census (bottom) datasets.}
    \label{fig:Realistic_F1_Scores}
    \end{figure}

\section{Discussion} \label{sec:discussion}
\subsection{Computational Cost Analysis} \label{sec:computational_cost}

To highlight the computational efficiency of our proposed method as compared to expensive shadow model-based membership inference attacks \cite{guepin2023synthetic}, we note the execution times of the latter across our datasets in Table~\ref{tab:shadow_runtime_all}. These durations encompass the complete shadow modeling pipeline -- repeated synthetic data generation and subsequent attack model training. Because the official implementation\footnote{https://github.com/imperial-aisp/MIA-synthetic} lacks native GPU support, evaluations were performed on a $20$-core Intel$^{\text{\textregistered}}$ Xeon$^{\text{\textregistered}}$ 4114 CPU setup. The results demonstrate a severe computational bottleneck even for classical generators (e.g., PrivBayes, Synthpop), exacerbated by the fact that these times reflect the classification of merely 10 test records. This inefficiency inherently stems from iteratively training generative models across multiple shadow splits, a process that scales poorly with dataset size.

\begin{table}[]
\centering
\caption{Execution time\tablefootnote{`h' denotes `hours' and `m' denotes `minutes.'} required for shadow model-based MIAs across different synthetic data generators and datasets.}
\label{tab:shadow_runtime_all}
\begin{adjustbox}{width=\linewidth}
\begin{tabular}{lcccc}
\toprule
Generator & MIMIC-IV & UK Census & Texas-100X & Nexoid \\
\midrule
PrivBayes & 9h 32m & 8h 13m & 8h 25m & 8h 53m \\
Synthpop & 7h 17m & 6h 20m & 6h 31m & 6h 51m \\
Bayesian Network & 5h 6m & 4h 42m & 4h 50m & 5h 6m \\
CTGAN & 20h 48m & 18h 10m & 18h 38m & 19h 35m \\
\bottomrule
\end{tabular}
\end{adjustbox}
\end{table}

Conversely, our KDE-based framework entirely circumvents shadow model training, relying instead on a single nearest-neighbor distance computation, followed by kernel density estimation. While high-dimensional pairwise distance calculations introduce a baseline overhead, we mitigate this bottleneck through GPU-parallelized execution and chunk-wise array processing. Table~\ref{table:kde_cpu_gpu_times} reports the execution times from KDE fitting to membership evaluation on CTGAN-generated datasets across multiple CPU and GPU configurations. Notably, evaluating thousands of test records (detailed in Table~\ref{table:training_test_distances}) requires trivial computation time compared to shadow modeling. The initial nearest-neighbor distance matrix generation incurs a 30 to 90-minute overhead, depending on the dataset size; however, this is a strictly one-time, cacheable expense that eliminates redundant computations in all subsequent evaluations (further details in supplementary Section~\ref{appendix:experiments}).

\begin{table}[!ht]
    \centering
    \caption{Time taken for implementing the KDE-based workflow on different GPU and CPU configurations, to compute the true distribution attack risk on CTGAN-generated synthetic datasets. The least times taken for each dataset are highlighted.}
    \begin{adjustbox}{width=\linewidth}
    \begin{tabular}{ccccc}
    \toprule
        Device & MIMIC-IV & UK Census & Texas-100X & Nexoid \\ \midrule
        GPU Tesla P100 $\times 1$ & \textbf{1.74s} & \textbf{17.2s} & \textbf{48.8s} & \textbf{18.9s} \\ 
        GPU RTX 2080Ti $\times 1$ & 3.91s & 43.5s & 1m 57s & 49.2s \\ 
        CPU Intel$^{\text{\textregistered}}$  Xeon$^{\text{\textregistered}}$ $\times 96$ & 7.4s & 1m 19s &	3m 26s & 1m 34s \\ 
        GPU Tesla T4 $\times 1$ & 7.5s & 1m 32s & 4m 10s & 1m 44s \\ 
        CPU Intel$^{\text{\textregistered}}$ Xeon$^{\text{\textregistered}}$ 4114 $\times 20$ & 26.9s & 5m 6s & 14m 56s & 6m 39s \\ \bottomrule
    \end{tabular}
    \end{adjustbox}
    \label{table:kde_cpu_gpu_times}
\end{table}

Under certain realistic assumptions, the proposed framework achieves an average speedup of approximately $\mathbf{200,000\times}$ over the state-of-the-art shadow modeling technique on identical CPU hardware (see Figure~\ref{fig:speedup} and supplementary Section~\ref{appendix:computational}). For data custodians requiring repeated MIA risk assessments across multiple datasets or generative models, the computational exhaustiveness of shadow models is often prohibitive. By contrast, our KDE-based approach provides a highly scalable, computationally lightweight alternative for post-generation risk assessment while preserving rigorous probabilistic risk characterization.

\begin{figure}[h!]
   \centering
   \includegraphics[width=0.8\linewidth]{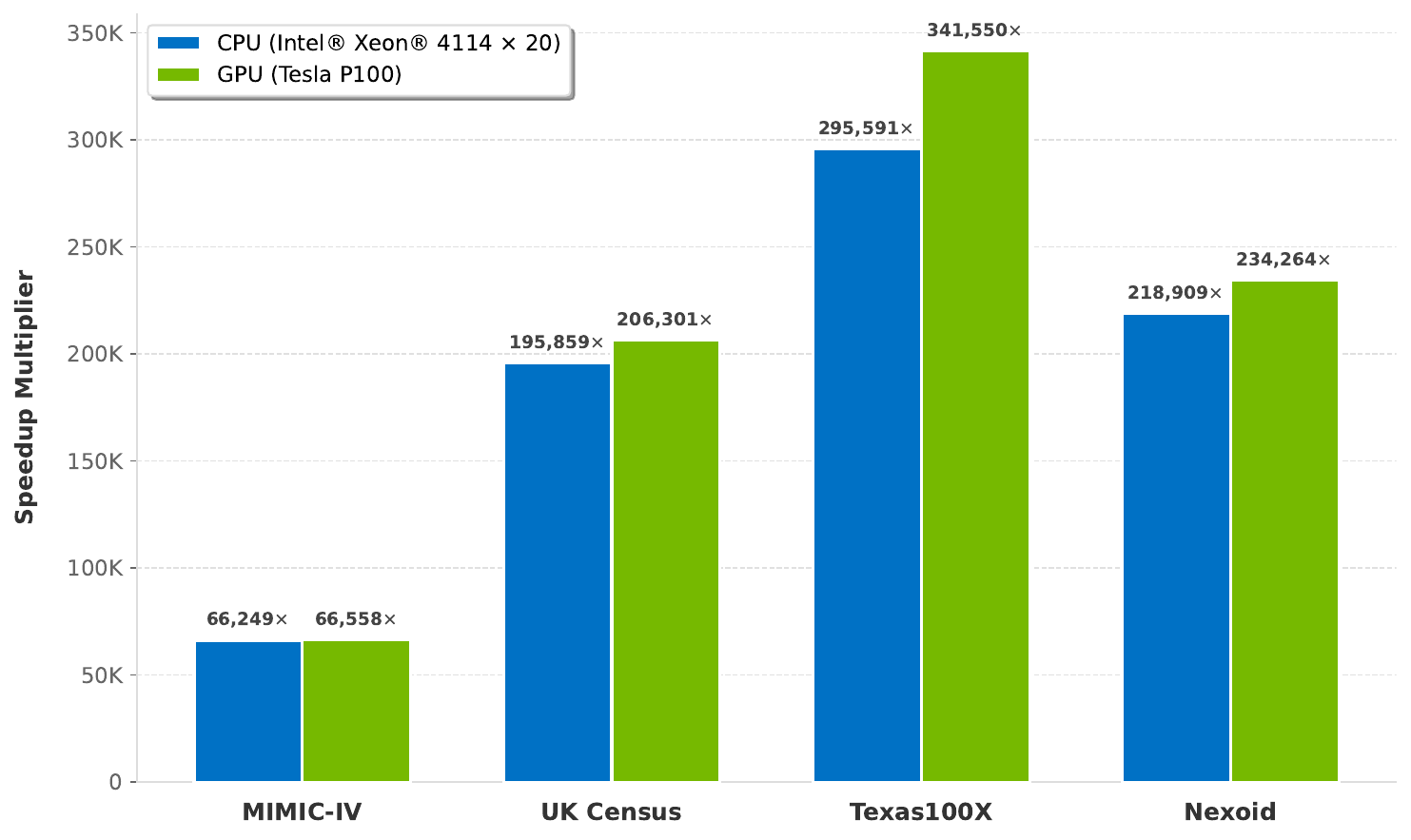}
   \caption{Speedup offered by the proposed method against shadow model-based MIA using CTGAN generator on CPU and GPU devices.}
    \label{fig:speedup}
    \end{figure}

\subsection{ROC Analysis} \label{discussion:ROC}
Our approach enables seamless ROC analysis, revealing critical limitations of average-case metrics like accuracy and AUC in evaluating membership inference attacks. For example, TVAE-generated UK Census data achieves only 49.97\% accuracy (below baseline), while DPGAN achieves 57.93\% (8 percentage points above baseline), suggesting limited risk. However, log-ROC analysis (Figure \ref{fig:UKCensus_UB_RawDist_ROC}) reveals TPRs of 0.1-1.0 at FPR=$10^{-6}$, indicating TPR up to $10^5$ times greater than FPR -- a substantial vulnerability masked by benign average metrics. Similarly, TVAE-generated Nexoid data shows marginally above-baseline performance (AUC=0.554, accuracy=53.77\%), yet achieves TPR=$10^{-4}$-$10^{-3}$ at FPR=$10^{-6}$ (supplementary Figure \ref{fig:Nexoid_UB_RawDist_ROC}), representing TPR over 100 times higher than FPR. From a privacy perspective, TPRs at low FPRs should not exceed $t=20$ times the FPR (adjustable based on use-case requirements), and metric selection should align with application-specific risk tolerance and privacy requirements.

This highlights the main limitation of the distance-based baseline method \cite{el2022MIA_synth}, which has been implemented in a lot of works in practice for evaluating MIA risk \cite{mendelevitch2021fidelity, choi2017generating, zhang2020ensuring, goncalves2020generation, chen2020gan, yan2021generating}. ROC analysis is critical to identify high-TPR at low-FPR regimes in order to quantify worst-case privacy leakage risks \cite{carlini2022MIA_ML_3}. Hence, our method offers a middle ground between highly expensive shadow model-based attacks and comparatively less powerful distance-based attacks, aiming to serve as an important tool for data custodians and practitioners.

\section{Conclusion}
We proposed a computationally efficient KDE-based framework for calculating membership disclosure risk in tabular synthetic datasets, by modeling nearest-neighbor distance distributions. Our method enables probabilistic membership inference without expensive shadow modeling, with further acceleration through distance calculations and KDE estimation on GPUs. The proposed method primarily functions as a \textbf{post-generation metric}, allowing data custodians to assess risk after synthesis using only the training data, synthetic dataset, and reference data from the same population (a subset of the real data not used in training). Through extensive experiments with four real-world datasets and six state-of-the-art synthetic data generators, we demonstrate that our method is (1) extremely computationally lightweight as compared to shadow model-based MIAs, and (2) stronger in terms of attack strength than baseline distance-based MIAs, thereby offering data holders a practical tool to assess membership inference risk.

\subsection*{Limitations and Future Directions}
The experiments conducted in this study use balanced attack datasets with equal number of member/non-member records, to evaluate attack performance improvement (or risk) with a well-established baseline risk (50\% in terms of attack accuracy). However, \cite{el2022MIA_synth} suggest setting the member proportions to $n/N$ (training size/population size) for realistic deployment scenarios. To this end, our implementation supports flexible partitioning and train-test-split parameters, offering practitioners with the flexibility of choosing the proportion of members and non-members for testing based on their requirements. Future work could explore: (1) relaxing balanced dataset assumptions to improve practical applicability, (2) investigating theoretical guarantees on distance-to-membership probability mappings, (3) extending the KDE-based methodology from tabular to general data (e.g., images), and (4) developing hybrid strategies combining KDE-based post-hoc assessment with lightweight shadow modeling or adversarial training for more robust membership disclosure risk evaluation.

\bibliographystyle{IEEEtran}
\bibliography{sample-base}

@String{Computer = "{IEEE} Computer" }

@String{Springer = "Springer-Verlag" }

@article{Gower1971,
    author = "J.C. Gower",
    title = "A General Coefficient of Similarity and Some of Its Properties",
    journal = "Biometrics",
    year = "1971",
    month = "April",
    volume = "27",
    number = "4",
    pages = "857--871"
}

@article{xu2019modeling,
  title={Modeling tabular data using conditional gan},
  author={Xu, Lei and Skoularidou, Maria and Cuesta-Infante, Alfredo and Veeramachaneni, Kalyan},
  journal={Advances in neural information processing systems},
  volume={32},
  year={2019}
}

@article{zhang2017privbayes,
  title={Privbayes: Private data release via bayesian networks},
  author={Zhang, Jun and Cormode, Graham and Procopiuc, Cecilia M and Srivastava, Divesh and Xiao, Xiaokui},
  journal={ACM Transactions on Database Systems (TODS)},
  volume={42},
  number={4},
  pages={1--41},
  year={2017},
  publisher={ACM New York, NY, USA}
}

@inproceedings{vardhan2020generating,
  title={Generating privacy-preserving synthetic tabular data using oblivious variational autoencoders},
  author={Vardhan, L Vivek Harsha and Kok, Stanley},
  booktitle={Proceedings of the Workshop on Economics of Privacy and Data Labor at the 37 th International Conference on Machine Learning},
  year={2020}
}

@article{ma2020vaem,
  title={Vaem: a deep generative model for heterogeneous mixed type data},
  author={Ma, Chao and Tschiatschek, Sebastian and Turner, Richard and Hern{\'a}ndez-Lobato, Jos{\'e} Miguel and Zhang, Cheng},
  journal={Advances in Neural Information Processing Systems},
  volume={33},
  pages={11237--11247},
  year={2020}
}

@inproceedings{choi2017generating,
  title={Generating multi-label discrete patient records using generative adversarial networks},
  author={Choi, Edward and Biswal, Siddharth and Malin, Bradley and Duke, Jon and Stewart, Walter F and Sun, Jimeng},
  booktitle={Machine learning for healthcare conference},
  pages={286--305},
  year={2017},
  organization={PMLR}
}

@inproceedings{shokri2017MIA_ML_2,
  title={Membership inference attacks against machine learning models},
  author={Shokri, Reza and Stronati, Marco and Song, Congzheng and Shmatikov, Vitaly},
  booktitle={2017 IEEE symposium on security and privacy (SP)},
  pages={3--18},
  year={2017},
  organization={IEEE}
}

@inproceedings{carlini2022MIA_ML_3,
  title={Membership inference attacks from first principles},
  author={Carlini, Nicholas and Chien, Steve and Nasr, Milad and Song, Shuang and Terzis, Andreas and Tramer, Florian},
  booktitle={2022 IEEE Symposium on Security and Privacy (SP)},
  pages={1897--1914},
  year={2022},
  organization={IEEE}
}

@article{el2022MIA_synth,
  title={Validating a membership disclosure metric for synthetic health data},
  author={El Emam, Khaled and Mosquera, Lucy and Fang, Xi},
  journal={JAMIA open},
  volume={5},
  number={4},
  pages={ooac083},
  year={2022},
  publisher={Oxford University Press}
}

@inproceedings{yale2019assessing,
  title={Assessing privacy and quality of synthetic health data},
  author={Yale, Andrew and Dash, Saloni and Dutta, Ritik and Guyon, Isabelle and Pavao, Adrien and Bennett, Kristin P},
  booktitle={Proceedings of the Conference on Artificial Intelligence for Data Discovery and Reuse},
  pages={1--4},
  year={2019}
}

@article{parzen1962estimation,
  title={On estimation of a probability density function and mode},
  author={Parzen, Emanuel},
  journal={The Annals of Mathematical Statistics},
  volume={33},
  number={3},
  pages={1065--1076},
  year={1962},
  publisher={JSTOR}
}

@article{giomi2023unified,
  title={{A Unified Framework for Quantifying Privacy Risk in Synthetic Data}},
  author={Giomi, Matteo and Boenisch, Franziska and Wehmeyer, Christoph and Tasn{\'a}di, Borb{\'a}la},
  journal={Proceedings on Privacy Enhancing Technologies},
  year={2023}
}

@article{lautrup2025syntheval,
  title={Syntheval: a framework for detailed utility and privacy evaluation of tabular synthetic data},
  author={Lautrup, Anton D and Hyrup, Tobias and Zimek, Arthur and Schneider-Kamp, Peter},
  journal={Data Mining and Knowledge Discovery},
  volume={39},
  number={1},
  pages={6},
  year={2025},
  publisher={Springer}
}

@inproceedings{pamisetty2024ehrgpt,
  title={{EHRGPT: Leveraging language models for generating synthetic health records}},
  author={Pamisetty, Giridhar and Batreddy, Subbareddy and Verma, Priya and Babu, Ch Sobhan},
  booktitle={2024 IEEE International Conference on Big Data (BigData)},
  pages={4579--4587},
  year={2024},
  organization={IEEE}
}

@inproceedings{van2023membership,
  title={{Membership Inference Attacks against Synthetic Data through Overfitting Detection}},
  author={van Breugel, Boris and Sun, Hao and Qian, Zhaozhi and van der Schaar, Mihaela},
  booktitle={International Conference on Artificial Intelligence and Statistics},
  pages={3493--3514},
  year={2023},
  organization={PMLR}
}

@article{lassila2026practical,
  title={Practical bayes-optimal membership inference attacks},
  author={Lassila, Marcus and {\"O}stman, Johan and Ngo, Khac-Hoang and Graell i Amat, Alexandre},
  journal={Advances in Neural Information Processing Systems},
  volume={38},
  pages={34278--34312},
  year={2026}
}

@article{li2026vid,
  title={{Vid-SME: Membership Inference Attacks against Large Video Understanding Models}},
  author={Li, Qi and Yu, Runpeng and Wang, Xinchao},
  journal={Advances in Neural Information Processing Systems},
  volume={38},
  pages={111572--111596},
  year={2026}
}

@inproceedings{rossi2025membership,
  title={Membership inference attacks on sequence models},
  author={Rossi, Lorenzo and Aerni, Michael and Zhang, Jie and Tram{\`e}r, Florian},
  booktitle={2025 IEEE Security and Privacy Workshops (SPW)},
  pages={98--110},
  year={2025},
  organization={IEEE}
}

@article{wu2024rethinking,
  title={{Rethinking Membership Inference Attacks Against Transfer Learning}},
  author={Wu, Cong and Chen, Jing and Fang, Qianru and He, Kun and Zhao, Ziming and Ren, Hao and Xu, Guowen and Liu, Yang and Xiang, Yang},
  journal={IEEE Transactions on Information Forensics and Security},
  volume={19},
  pages={6441--6454},
  year={2024},
  publisher={IEEE}
}

@article{he2025membership,
  title={Membership inference attacks on recommender system: A survey},
  author={He, Jiajie and Chen, Xintong and Fang, Xinyang and Chen, Min-Chun and Gu, Yuechun and Chen, Keke},
  journal={arXiv preprint arXiv:2509.11080},
  year={2025}
}

@inproceedings{he2025towards,
  title={{Towards Label-Only Membership Inference Attack against Pre-trained Large Language Models}},
  author={He, Yu and Li, Boheng and Liu, Liu and Ba, Zhongjie and Dong, Wei and Li, Yiming and Qin, Zhan and Ren, Kui and Chen, Chun},
  booktitle={34th USENIX Security Symposium (USENIX Security 25)},
  pages={1609--1628},
  year={2025}
}

@inproceedings{hu2025membership,
  title={{Membership Inference Attacks Against Vision-Language Models}},
  author={Hu, Yuke and Li, Zheng and Liu, Zhihao and Zhang, Yang and Qin, Zhan and Ren, Kui and Chen, Chun},
  booktitle={34th USENIX Security Symposium (USENIX Security 25)},
  pages={1589--1608},
  year={2025}
}

@inproceedings{guepin2023synthetic,
  title={Synthetic is all you need: removing the auxiliary data assumption for membership inference attacks against synthetic data},
  author={Gu{\'e}pin, Florent and Meeus, Matthieu and Cre{\c{t}}u, Ana-Maria and de Montjoye, Yves-Alexandre},
  booktitle={European Symposium on Research in Computer Security},
  pages={182--198},
  year={2023},
  organization={Springer}
}

@misc{synthcity,
  doi = {10.48550/ARXIV.2301.07573},
  url = {https://arxiv.org/abs/2301.07573},
  author = {Qian, Zhaozhi and Cebere, Bogdan-Constantin and van der Schaar, Mihaela},
  title = {Synthcity: facilitating innovative use cases of synthetic data in different data modalities},
  year = {2023},
  copyright = {Creative Commons Attribution 4.0 International}
}

@article{yoon2020anonymization,
  title={Anonymization through data synthesis using generative adversarial networks (ADS-GAN)},
  author={Yoon, Jinsung and Drumright, Lydia N and Van Der Schaar, Mihaela},
  journal={IEEE journal of biomedical and health informatics},
  volume={24},
  number={8},
  pages={2378--2388},
  year={2020},
  publisher={IEEE}
}

@article{xie2018differentially,
  title={Differentially private generative adversarial network},
  author={Xie, Liyang and Lin, Kaixiang and Wang, Shu and Wang, Fei and Zhou, Jiayu},
  journal={arXiv preprint arXiv:1802.06739},
  year={2018}
}

@article{young2009using,
  title={Using Bayesian networks to create synthetic data},
  author={Young, Jim and Graham, Patrick and Penny, Richard},
  journal={Journal of Official Statistics},
  volume={25},
  number={4},
  pages={549--567},
  year={2009},
  publisher={Statistics Sweden (SCB)}
}

@article{zhang2020ensuring,
  title={Ensuring electronic medical record simulation through better training, modeling, and evaluation},
  author={Zhang, Ziqi and Yan, Chao and Mesa, Diego A and Sun, Jimeng and Malin, Bradley A},
  journal={Journal of the American Medical Informatics Association},
  volume={27},
  number={1},
  pages={99--108},
  year={2020},
  publisher={Oxford University Press}
}

@article{goncalves2020generation,
  title={Generation and evaluation of synthetic patient data},
  author={Goncalves, Andre and Ray, Priyadip and Soper, Braden and Stevens, Jennifer and Coyle, Linda and Sales, Ana Paula},
  journal={BMC medical research methodology},
  volume={20},
  pages={1--40},
  year={2020},
  publisher={Springer}
}

@article{koning2005m3,
  title={{The M3 competition: Statistical tests of the results}},
  author={Koning, Alex J and Franses, Philip Hans and Hibon, Michele and Stekler, Herman O},
  journal={{International Journal of Forecasting}},
  volume={21},
  number={3},
  pages={397--409},
  year={2005},
  publisher={Elsevier}
}

@inproceedings{chen2020gan,
  title={{GAN-Leaks: A Taxonomy of Membership Inference Attacks against Generative Models}},
  author={Chen, Dingfan and Yu, Ning and Zhang, Yang and Fritz, Mario},
  booktitle={Proceedings of the 2020 ACM SIGSAC conference on computer and communications security},
  pages={343--362},
  year={2020}
}

@inproceedings{yan2021generating,
  title={Generating electronic health records with multiple data types and constraints},
  author={Yan, Chao and Zhang, Ziqi and Nyemba, Steve and Malin, Bradley A},
  booktitle={AMIA annual symposium proceedings},
  volume={2020},
  pages={1335},
  year={2021}
}

@inproceedings{kotelnikov2023tabddpm,
  title={Tabddpm: Modelling tabular data with diffusion models},
  author={Kotelnikov, Akim and Baranchuk, Dmitry and Rubachev, Ivan and Babenko, Artem},
  booktitle={International Conference on Machine Learning},
  pages={17564--17579},
  year={2023},
  organization={PMLR}
}

@article{johnson2023mimic,
  title={{MIMIC-IV, a freely accessible electronic health record dataset}},
  author={Johnson, Alistair EW and Bulgarelli, Lucas and Shen, Lu and Gayles, Alvin and Shammout, Ayad and Horng, Steven and Pollard, Tom J and Hao, Sicheng and Moody, Benjamin and Gow, Brian and others},
  journal={Scientific data},
  volume={10},
  number={1},
  pages={1},
  year={2023},
  publisher={Nature Publishing Group UK London}
}

@article{johnson2020mimic,
  title={Mimic-iv},
  author={Johnson, Alistair and Bulgarelli, Lucas and Pollard, Tom and Horng, Steven and Celi, Leo Anthony and Mark, Roger},
  journal={PhysioNet. Available online at: https://physionet. org/content/mimiciv/1.0/(accessed August 23, 2021)},
  pages={49--55},
  year={2020}
}

@misc{ONS2011Census,
  author       = {{Office for National Statistics}},
  title        = {Census Microdata Teaching Files (2011)},
  year         = {2011},
  howpublished = {\url{https://www.ons.gov.uk/census/2011census/2011censusdata/
                       censusmicrodata/microdatateachingfile}},
  note         = {Accessed: 2025-06-05}
}

@misc{texas2006inpatient,
  title = {Texas Hospital Inpatient Discharge Public Use Data File, [Quarters 1-4, 2006]},
  author = {{Texas Department of State Health Services}},
  year = {2022},
  note = {[June 21, 2022]},
  howpublished = {\url{https://www.dshs.texas.gov/centers-health-statistics/health-care-data-collection/texas-inpatient-public-use-data-file-pudf}},
  institution = {Texas Department of State Health Services, Austin, Texas}
}

@misc{nexoid2020covid19,
  author       = {{Nexoid Ltd}},
  title        = {Nexoid COVID-19 Dataset},
  year         = {2020},
  howpublished          = {https://www.covid19survivalcalculator.com/download},
  note         = {Accessed: 2025-06-06}
}

@inproceedings{stadler2022synthetic,
  title={Synthetic data--anonymisation groundhog day},
  author={Stadler, Theresa and Oprisanu, Bristena and Troncoso, Carmela},
  booktitle={31st USENIX Security Symposium (USENIX Security 22)},
  pages={1451--1468},
  year={2022}
}

@article{mendelevitch2021fidelity,
  title={Fidelity and privacy of synthetic medical data},
  author={Mendelevitch, Ofer and Lesh, Michael D},
  journal={arXiv preprint arXiv:2101.08658},
  year={2021}
}

@inproceedings{ankan2015pgmpy,
  title={pgmpy: Probabilistic Graphical Models using Python.},
  author={Ankan, Ankur and Panda, Abinash},
  booktitle={SciPy},
  pages={6--11},
  year={2015}
}

@inproceedings{shokri2017membership,
  title={Membership inference attacks against machine learning models},
  author={Shokri, Reza and Stronati, Marco and Song, Congzheng and Shmatikov, Vitaly},
  booktitle={2017 IEEE symposium on security and privacy (SP)},
  pages={3--18},
  year={2017},
  organization={IEEE}
}

@inproceedings{zhang2023generative,
  title={Generative table pre-training empowers models for tabular prediction},
  author={Zhang, Tianping and Wang, Shaowen and Yan, Shuicheng and Jian, Li and Liu, Qian},
  booktitle={Proceedings of the 2023 conference on empirical methods in natural language processing},
  pages={14836--14854},
  year={2023}
}

@inproceedings{zhao2025tabula,
  title={{Tabula: Harnessing language models for tabular data synthesis}},
  author={Zhao, Zilong and Birke, Robert and Chen, Lydia Y},
  booktitle={Pacific-Asia Conference on Knowledge Discovery and Data Mining},
  pages={247--259},
  year={2025},
  organization={Springer}
}

@inproceedings{
borisov2023language,
title={{Language Models are Realistic Tabular Data Generators}},
author={Vadim Borisov and Kathrin Sessler and Tobias Leemann and Martin Pawelczyk and Gjergji Kasneci},
booktitle={The Eleventh International Conference on Learning Representations },
year={2023},
url={https://openreview.net/forum?id=cEygmQNOeI}
}

@article{gulati2023tabmt,
  title={{TabMT: Generating tabular data with masked transformers}},
  author={Gulati, Manbir and Roysdon, Paul},
  journal={Advances in Neural Information Processing Systems},
  volume={36},
  pages={46245--46254},
  year={2023}
}

\appendices
\renewcommand{\thefigure}{S.\arabic{figure}}
\renewcommand{\thetable}{S.\arabic{table}}
\section{Proof for \Cref{proposition:prob_justification}} \label{appendix:proof}
\begin{proof}
    Assume that the prior probabilities of the query record being a member and non-member are equal, i.e., $\mathbf{P}(\text{member}) = \mathbf{P}(\text{non-member}) = 0.5$. From Bayes' Theorem, we have 
    \begin{equation} \label{eq:Bayes}
        \mathbf{P}(\text{member}|d) = \frac{\mathbf{P}(d|\text{member}) \times \mathbf{P}(\text{member})}{ \mathbf{P}(d)}
    \end{equation}
    By the Theorem of Total Probability, it follows that the probability of observing the distance $d$ is given by:
    \begin{align*}
    \mathbf{P}(d) &= \mathbf{P}(d|\text{member}) \mathbf{P}(\text{member}) \\
    &\qquad+ \mathbf{P}(d|\text{non-member}) \mathbf{P}(\text{non-member}) \\
    &= \mathbf{P}(\text{member}) \ [\mathbf{P}(d|\text{member}) + \mathbf{P}(d|\text{non-member})]
    \end{align*}
    Putting this in the denominator of \eqref{eq:Bayes}, we get that $\mathbf{P}(\text{member}|d) =$
    {\small
    \begin{align*}
    &= \frac{\mathbf{P}(d|\text{member}) \times \mathbf{P}(\text{member})}{\mathbf{P}(\text{member}) \ [\mathbf{P}(d|\text{member}) + \mathbf{P}(d|\text{non-member})] }\\    
    &= \frac{\mathbf{P}(d|\text{member})}{\mathbf{P}(d|\text{member}) + \mathbf{P}(d|\text{non-member})}\\
    &\approx \frac{KDE_{\text{member}}(d)}{KDE_{\text{member}}(d) + KDE_{\text{non-member}}(d)}
    \end{align*}}
    Since, $KDE_{\text{member}}(d)$ and $KDE_{\text{non-member}}(d)$ approximate the probabilities $\mathbf{P}(d|\text{member})$ and $\mathbf{P}(d|\text{non-member})$ respectively.
\end{proof}

\section{Datasets} \label{appendix:datasets}
This section elaborates on the four real-world datasets used in our experiments.
\subsection{MIMIC-IV}
The Medical Information Mart for Intensive Care IV (MIMIC-IV) dataset is a structured and de-identified electronic health record (EHR) resource released in July 2024 \cite{johnson2020mimic, johnson2023mimic}. The dataset contains the information of patients treated at the Beth Israel Deaconess Medical Center (BIDMC) in Boston, Massachusetts, USA, between 2008 and 2019. The original dataset used to train the generation models was preprocessed by adopting the method described in \cite{pamisetty2024ehrgpt}. In their original work, this preprocessing resulted in a consolidated master health record table comprising about $430,000$ records and 147 features per record. As the datasets in this study are for illustrative purposes only, we randomly sample $351{,}231$ records and select 20 important features to form the real (training) dataset for computational efficiency. We generate a synthetic dataset of the same size comprising the same features.

\subsection{UK Census}
The UK Census data is the 2011 Census Microdata Teaching File \cite{ONS2011Census}, which is a publicly available dataset published by the Office for National Statistics. It contains a 1\% random sample of the 2011 Census output for England and Wales, totaling $569{,}741$ records. The dataset includes 17 categorical attributes, covering demographic, socio-economic, and housing-related variables. All data are anonymized and represent individual-level information. It is commonly used for teaching, research, and statistical modeling tasks.

\subsection{Texas-100X}
The Texas-100X dataset \cite{texas2006inpatient} comprises $925{,}128$ anonymized inpatient records from 441 hospitals across Texas in 2006. Each record includes demographic details (age, gender, race, ethnicity) and medical information (hospitalization duration, admission type and source, admitting diagnosis, patient status, medical charges, and principal surgical procedure). The primary machine learning task involves predicting one of 100 surgical procedures based on patient data. The dataset is intended for academic research in machine learning and privacy analysis, with strict prohibitions against attempts to re-identify individuals.

\subsection{Nexoid COVID-19}
The Nexoid COVID-19 dataset \cite{nexoid2020covid19} comprises 1,023,426 anonymized records and 60 features. The dataset encompasses a wide range of variables, including demographics, geographic information, health conditions, behaviors (e.g., public transport usage, healthcare employment), and socioeconomic factors. To ensure privacy, Nexoid implemented measures such as age banding, randomized timestamps, and slight perturbations of geolocation data. During our experiments, columns with more than 67\% missing values were excluded from the dataset, and a random sample of $600{,}000$ records was selected to ensure computational tractability. The resulting dataset comprises 42 features, with missing values imputed using the most frequent category for categorical attributes and the mode for numerical attributes.

\section{Synthetic Data Generators} \label{appendix:generators}
\subsection{CTGAN} \label{methods:ctgan}
The Conditional Tabular Generative Adversarial Network (CTGAN) \cite{xu2019modeling} generates high-fidelity synthetic tabular data by introducing domain-specific techniques to handle mixed continuous and categorical variables. Because standard normalization assumes unimodal distributions, often failing on complex real-world data, CTGAN resolves this via \textit{mode-specific normalization}. It fits a Gaussian Mixture Model (GMM) to identify underlying modes for each continuous column, representing each value as a tuple: its normalized scalar within a specific Gaussian mode and a one-hot encoding of that mode. To model categorical variables and mitigate class imbalance, CTGAN employs a conditional generation mechanism driven by \textit{log-frequency sampling}. During training, categories are sampled proportional to the logarithm of their inverse frequency, effectively up-sampling rare classes. A \textit{conditional vector} (a one-hot encoding of the sampled category) is fed to both the generator and discriminator, guiding the model to learn balanced, category-specific distributions. Building on the standard GAN framework, the generator $G$ processes a noise vector $z$ and the conditional vector to output normalized continuous values and categorical logits. The discriminator $D$ evaluates real or synthetic samples alongside this conditional vector. The architecture optimizes a combined loss function -- cross-entropy for categorical outputs and Wasserstein or $L2$ loss for continuous components. To stabilize training and penalize mode collapse, CTGAN incorporates the PacGAN framework, which evaluates multiple samples simultaneously. Finally, the network's synthetic outputs are post-processed to reconstruct the original tabular format. Continuous variables are denormalized using the inverse GMM transformation, while categorical values are recovered via an argmax operation over the generated logits.

\subsection{ADS-GAN} \label{methods:adsgan}
Anonymization through Data Synthesis using Generative Adversarial Networks (ADS-GAN) \cite{yoon2020anonymization} synthesizes confidential data while strictly bounding individual-level identifiability. Built on a modified conditional GAN (cGAN), ADS-GAN dynamically optimizes the set of conditioning variables for each generated sample. This localized, per-sample conditioning enhances overall data fidelity while effectively mitigating re-identification risks. To enforce privacy, the model integrates $\epsilon$-identifiability, a metric penalizing the proportion of synthetic records that are excessively similar to real records based on a weighted Euclidean distance (a standard also implemented in SynthEval \cite{lautrup2025syntheval}). This privacy constraint is directly embedded into the adversarial training objective:
{\scriptsize
$$\min_G \max_D \left[ \mathbb{E}_{x \sim P_X}[D(x)] - \mathbb{E}_{\hat{x} \sim P_{\hat{X}}}[D(\hat{x})] + \lambda \cdot \mathbb{E}_{x, \hat{x}}[-\Vert{}w \cdot (x - \hat{x})\Vert{}] \right],$$}
where $D$ and $G$ denote the multi-layer perceptron (MLP) discriminator and generator, $w$ represents a feature-specific inverse-entropy weight vector reflecting identifiability risk, and $\lambda$ dictates the trade-off between privacy and distributional fidelity. To stabilize optimization and circumvent mode collapse, ADS-GAN adopts the Wasserstein GAN with Gradient Penalty (WGAN-GP) framework. The generator produces synthetic representations from a combination of noise and real inputs, followed by standard post-processing to accurately reconstruct the original dataset's categorical and continuous structures.

\subsection{DPGAN}
Differentially Private Generative Adversarial Network (DPGAN) \cite{xie2018differentially} synthesizes data with formal $(\epsilon, \delta)$-differential privacy (DP) guarantees by integrating calibrated noise into the training dynamics of a Wasserstein GAN (WGAN). To rigorously bound the influence of any single training record, DPGAN clips the discriminator's gradients to control sensitivity before injecting calibrated Gaussian noise during RMSProp optimization. By building upon the WGAN architecture, the framework utilizes the Wasserstein distance to ensure stable convergence and mitigate mode collapse, even under noisy training conditions. The required noise scale ($\sigma_n$) is analytically determined via the moments accountant technique, which tracks the accumulated privacy loss across iterations based on the sampling rate, the number of discriminator updates, and the target privacy budget $\epsilon$. This approach yields significantly tighter bounds than standard composition theorems, enabling a superior trade-off between privacy and data utility. DPGAN leverages the post-processing invariance property of differential privacy, which mathematically guarantees that a DP-compliant discriminator inherently produces a DP-compliant generator. Through this formulation, the model successfully learns complex data distributions while maintaining strict, provable privacy bounds.

\subsection{TabDDPM}
TabDDPM \cite{kotelnikov2023tabddpm} adapts the denoising diffusion probabilistic model (DDPM) framework to synthesize tabular datasets, employing a hybrid diffusion strategy to handle mixed data types. To manage heterogeneous features, numerical variables are first normalized via a Gaussian quantile transformation and subjected to Gaussian diffusion, while categorical variables are one-hot encoded and independently corrupted through multinomial diffusion (uniform categorical noise). These preprocessed representations are concatenated to form the initial input state $x_0$. The reverse (denoising) process is parameterized by a MLP that takes the sum of the embedded input vector, sinusoidal timestep embeddings, and optional class embeddings. At each timestep, the MLP predicts the added Gaussian noise $\epsilon$ for numerical features and outputs class logits for categorical features. The generative process is inherently class-conditional for classification tasks, whereas for regression, the continuous target is simply modeled as an additional numerical feature. TabDDPM optimizes a composite objective function that integrates mean-squared error for the numerical variables with Kullback–Leibler divergence for the $C$ categorical variables:
$$\mathcal{L}_{\text{TabDDPM}} = \mathcal{L}_{\text{Gaussian}} + \frac{1}{C} \sum_{i=1}^{C} \mathcal{L}_{\text{Multinomial}}^{(i)}.$$
Extensive benchmarking against established generative models, including CTGAN, TVAE, and CTABGAN+, demonstrates that TabDDPM consistently achieves superior performance in maintaining feature-wise distributions, preserving complex correlation structures, and maximizing downstream machine learning utility \cite{kotelnikov2023tabddpm}.

\subsection{TVAE}
The Tabular Variational Autoencoder (TVAE) \cite{xu2019modeling} adapts the classical VAE architecture for heterogeneous tabular data by learning a continuous latent representation of the data distribution through probabilistic inference. To accommodate mixed data types, TVAE utilizes mode-specific normalization for continuous variables, representing each as a tuple $(\alpha_{i,j}, \beta_{i,j})$, where $\alpha_{i,j}$ is a normalized Gaussian scalar and $\beta_{i,j}$ is a one-hot mode indicator. Categorical variables are standard one-hot encoded vectors $d_{i,j}$. The architecture consists of an encoder $q_{\phi}(z_j \vert{} r_j)$ and a decoder $p_{\theta}(r_j \vert{} z_j)$. The encoder maps an input record $r_j$ through two hidden layers (with ReLU activations) to parameterize a multivariate Gaussian latent space, applying the reparameterization trick to output $q_{\phi}(z_j \vert{} r_j) = \mathcal{N}(\mu, \text{diag}(\sigma^2))$. Conversely, the decoder reconstructs the original variables by generating a joint distribution from the latent variable $z_j$. It applies a fully connected layer with a tanh activation to output a mean $\bar{\alpha}_{i,j}$ and variance $\delta_i$ for continuous scalars, while outputting softmax distributions for both categorical variables and mode indicators. The resulting joint predictive distribution is:
$$p_{\theta}(r_j \vert{} z_j) = \prod_{i=1}^{N_c} \mathcal{N}(\alpha_{i,j}; \bar{\alpha}_{i,j}, \delta_i) \cdot \prod_{i=1}^{N_c} \text{Cat}(\beta_{i,j}) \cdot \prod_{i=1}^{N_d} \text{Cat}(d_{i,j}).$$
TVAE is trained via backpropagation using the Adam optimizer to minimize the negative evidence lower bound (ELBO), effectively balancing reconstruction fidelity with latent space regularization:
$$\mathcal{L}_{\text{ELBO}} = \mathbb{E}_{q_{\phi}(z_j \vert{} r_j)}[\log p_{\theta}(r_j \vert{} z_j)] - D_{KL}(q_{\phi}(z_j \vert{} r_j) \Vert{} p(z_j)),$$
where $p(z_j)$ serves as the standard normal prior.

\subsection{Bayesian Network}
Bayesian networks \cite{young2009using} model the joint probability distribution of a dataset using a directed acyclic graph (DAG) to encode conditional independencies. Network structures are either domain-specified or learned algorithmically using scoring metrics like the Bayesian Information Criterion (BIC) or Bayes factors. Each node is governed by a Conditional Probability Distribution (CPD), typically a multinomial distribution with a Dirichlet prior for discrete data, allowing the full joint distribution to be factorized along the DAG. Following parameter learning (updating priors with empirical data), synthetic datasets are generated via topological sampling. Variables are drawn sequentially from root to leaf, conditioned on their parents' previously sampled values. This methodology was notably utilized by Young et al. \cite{young2009using} to generate multiple synthetic cohorts for privacy-preserving statistical inference.

To address structural uncertainty, Bayesian Model Averaging (BMA) can be applied by synthesizing data across various network structures sampled from the posterior. However, because the posterior typically concentrates on a few dominant models, BMA's practical benefit is often marginal \cite{young2009using}. Instead, synthetic data quality relies heavily on calibrating the imaginary sample size (prior weight). An excessively large prior can distort inference, whereas a small prior risks underfitting rare subgroups; balancing this parameter is crucial for optimizing the trade-off between data utility and confidentiality.

For implementation, these models are supported by libraries such as \textit{deal} (R) and \textit{pgmpy} (Python) \cite{ankan2015pgmpy}. In this study, we generate Bayesian network data utilizing the Synthcity framework \cite{synthcity}, which leverages \textit{pgmpy} as its computational backend for graph construction and sampling.

\section{Experimental Details} \label{appendix:experiments}
We leveraged two GPUs -- NVIDIA Tesla P100 (16 GB) and NVIDIA GeForce RTX 2080Ti (11.2 GB) to train the models and generate synthetic datasets. We use \texttt{n\_iter} values ranging from $200$ to $1000$ to train the generators (Table \ref{table:n_iter_values}). In GAN-based models, \texttt{n\_iter} denotes the maximum number of training iterations applied to the generator. In TVAE, \texttt{n\_iter} is the maximum number of iterations in the encoder, while for TabDDPM, it is the maximum number of epochs for training. All other model-specific parameters such as number of hidden layers, number of iterations in the discriminator (or decoder), activation functions, and learning rates were set to the default values of Synthcity \cite{synthcity}. The time taken in synthetic data generation across all datasets and generators are presented in \Cref{table:training_times}.

\begin{table}[ht!]
    \centering
    \caption{{\small \texttt{n\_iter} values used for training different generators across different datasets. The values highlighted in bold were determined by early stopping.}}
    \begin{adjustbox}{width=\linewidth}
    \begin{tabular}{ccccc}
    \toprule
        ~ & MIMIC-IV & UK Census & Texas-100X & Nexoid \\ \midrule
        CTGAN & 200  & 500 & 200 & 200 \\ 
        ADS-GAN & 200  & 500 & 200 & 200 \\ 
        DPGAN & 200 & 200 & 200 & 200 \\ 
        TabDDPM & 500 & 1000 & 500 & 500 \\ 
        TVAE & 200 & 500 & 500 & \textbf{450} \\ 
        Bayesian Network & 1000 & 1000 & 1000 & 1000 \\ \bottomrule
    \end{tabular}
    \end{adjustbox}
    \label{table:n_iter_values}
\end{table}

\begin{table}[!ht]
    \centering
    \caption{{\small Generation times incurred by various models across different datasets. The models corresponding to the training times highlighted in bold were trained on a NVIDIA Tesla P100 GPU, while the rest were trained on a RTX 2080 Ti GPU.}}
    \begin{adjustbox}{width=\linewidth}
    \begin{tabular}{ccccc}
    \toprule
        ~ & MIMIC-IV & UK Census & Texas-100X & Nexoid \\ \midrule
        CTGAN & \textbf{2h 51m 5s} & 8h 45m 9s & 5h 23m 29s & 10h 46min 52s \\ 
        ADS-GAN & \textbf{2h 56m 24s} & 10h 26m 41s & 5h 15m 14s & 10h 42min 48s \\ 
        DPGAN & \textbf{3h 27m 31s} & 5h 51m 15s & \textbf{3h 36m 53s} & 11h 25min 22s \\ 
        TabDDPM & \textbf{1h 32m 26s} & 5h 58m & 2h 26m 46s & \textbf{2h 1m 2s} \\ 
        TVAE & \textbf{2h 10m 5s} & 8h 21m 51s & 8h 37m 14s & 14h 58m 16s \\ 
        Bayesian Network & 6m & 7m & 15m 21s & 1h 27m \\ \bottomrule
    \end{tabular}
    \end{adjustbox}
    \label{table:training_times}
\end{table}

\subsection{Nearest-Neighbor Distance Computation} \label{appendix:Gower}
Computing nearest-neighbor distances between the attack and synthetic datasets is a fundamental computational step in quantifying membership inference risk using our method. While previous implementations of distance-based methods frequently employ Hamming \cite{el2022MIA_synth, mendelevitch2021fidelity} or Euclidean metrics, we adopt Gower's distance \cite{Gower1971}, aligning with risk evaluation frameworks like Anonymeter \cite{giomi2023unified}, due to its robustness in handling mixed-type tabular data. Specifically, numerical features are evaluated using the $L_1$ (Manhattan) distance following min-max or range normalization. This bounds the contribution of each numerical dimension to the $[0,1]$ interval, preserving comparability across features with disparate scales. Conversely, categorical attributes are evaluated via exact matching, assigning a distance of $0$ for matches and $1$ otherwise.

For large-scale datasets, brute-force distance computation is prohibitively expensive. For instance, evaluating our Nexoid dataset sample (600,000 total records across 42 features) requires computing pairwise distances between 300,000 attack records (training and unseen) and 300,000 synthetic records, yielding $90$ billion distinct operations. While standard CPU partitioning mitigates this overhead, it remains bottlenecked by core availability. To resolve this, we leverage GPU-accelerated parallel processing. We partition the real and unseen datasets into chunks, distributing them across a multi-GPU architecture using the Dask library and executing vectorized numerical operations via CuPy arrays. Computations were distributed across two NVIDIA Tesla T4 and four NVIDIA RTX 2080Ti GPUs, with the latter yielding superior throughput due to higher core counts and memory capacity. Using this infrastructure, we efficiently extract the nearest-neighbor distance for each record in $D_{\text{attack}}$. These distances, paired with their ground-truth membership labels, populate the $D_{\text{dists}}$ table, which is subsequently partitioned into balanced training and testing subsets containing equal proportions of members and non-members.

\subsection{KDE Fitting}
At the core of our methodology is the independent fitting of Kernel Density Estimators (KDEs) to member and non-member distance distributions. While fitting a one-dimensional KDE is computationally trivial, evaluating density estimates across vast test sets poses a significant bottleneck. Because standard Python libraries (e.g., Scikit-learn, SciPy) lack native GPU acceleration for KDE operations, data custodians typically must resort to partitioning test sets for CPU-based parallelization. To circumvent this limitation and maximize computational efficiency, we utilize a generic, lightweight GPU-enabled KDE class utilizing CuPy. Requiring only minimal dependencies (CuPy and NumPy), this custom implementation natively supports both KDE fitting and large-scale density evaluation, and serves as the computational engine for our primary experiments. The end-to-end risk quantification workflow proceeds as follows: (1) initializing and partitioning distance arrays into training and testing subsets; (2) fitting KDEs to the training distances; (3) evaluating densities for the test points to derive membership probabilities; and (4) applying probability thresholds to output classification metrics (accuracy and F1 score). Execution time for this pipeline scales proportionally with the number of test points. To benchmark practical deployment costs, Table \ref{table:kde_cpu_gpu_times} in the main text reports the execution time for this complete workflow on CTGAN synthetic data across all four real-world datasets, comparing various GPU and CPU configurations. The corresponding dataset sizes for KDE fitting and evaluation are detailed in Table \ref{table:training_test_distances} in the main text.


To accommodate strict VRAM limits, test points were evaluated in tunable batches; a standard batch size of 1,000 was applied, with a minor reduction to 700 for the Texas-100X dataset on the RTX 2080Ti. CPU benchmarks utilized two Intel$^{\text{\textregistered}}$ Xeon$^{\text{\textregistered}}$ architectures (a 20-core system with 126 GB RAM and a 96-core system with 330 GB RAM), with peak memory consumption remaining under 10 GB. The empirical results demonstrate extreme performance gains via GPU acceleration. Notably, on the massive Texas-100X dataset, the Tesla P100 completed the entire workflow in just 48.8 seconds -- an approximate 94.55\% reduction in computation time compared to the 14 minutes and 56 seconds required by the 20-core CPU setup.

\section{Comparison of Computational Cost} \label{appendix:computational}
This section elaborates on the computational cost analysis introduced in Section \ref{sec:computational_cost} of the main text. For our proposed framework, pairwise nearest-neighbor distances across all datasets and generative models were pre-computed via a parallelized mechanism utilizing four NVIDIA RTX 2080Ti GPUs. Depending on dataset dimensionality and size, this initial phase required between 30 and 90 minutes. In contrast, the shadow model-based MIA benchmark \cite{guepin2023synthetic} was evaluated using a restricted subset of only 10 test records for membership classification, with the resulting execution times detailed in Table \ref{tab:shadow_runtime_all} (main text). Because the official implementation of \cite{guepin2023synthetic} lacks native GPU support, these baseline evaluations were confined to a 20-core CPU environment.

Operating under an identical hardware configuration, Table \ref{table:kde_cpu_gpu_times} (main text) reports the end-to-end execution times for our proposed method when evaluating CTGAN-generated synthetic data across vastly larger test cohorts, ranging from 48,000 (MIMIC-IV) to 277,540 (Texas-100X) records. To facilitate a direct performance comparison, we consider the following assumptions:
\begin{itemize}
    \item The computational overhead of the shadow model-based MIA scales linearly with the test set size. For example, the $1248.36$ minutes required to evaluate 10 test records in the MIMIC-IV dataset linearly projects to approximately 4,161 days for the full 48,000-record test set.  
    \item The preliminary distance computation step in our proposed framework imposes a strict maximum overhead of 90 minutes.
\end{itemize}

Grounded in these assumptions and standardizing on CTGAN as the underlying generator, Table \ref{table:speedup} quantifies the magnitude of the speedup multiplier achieved by the proposed method relative to \cite{guepin2023synthetic} across all four datasets, evaluated under both CPU and GPU hardware configurations.

\begin{table}[!ht]
    \centering
    \caption{{\small Speedup multipler of the proposed method vs. shadow-model based MIA, taking CTGAN as the synthetic data generator.}}
    \begin{adjustbox}{width=\linewidth}
    \begin{tabular}{ccccc}
    \toprule
        ~ & MIMIC-IV & UK Census & Texas100X & Nexoid \\ \midrule
        CPU Intel$^{\text{\textregistered}}$ Xeon$^{\text{\textregistered}}$ 4114 $\times 20$ & $66,249\times$ & $195,859\times$ & $295,591\times$ & $218,909\times$ \\ 
        GPU Tesla P100 & $66,558\times$ & $206,301\times$ & $341,550\times$ & $234,264\times$ \\ \bottomrule
    \end{tabular}
    \end{adjustbox}
    \label{table:speedup}
\end{table}

\section{Additional Figures}



\begin{figure}[h!]
    \centering
    \includegraphics[width=\linewidth]{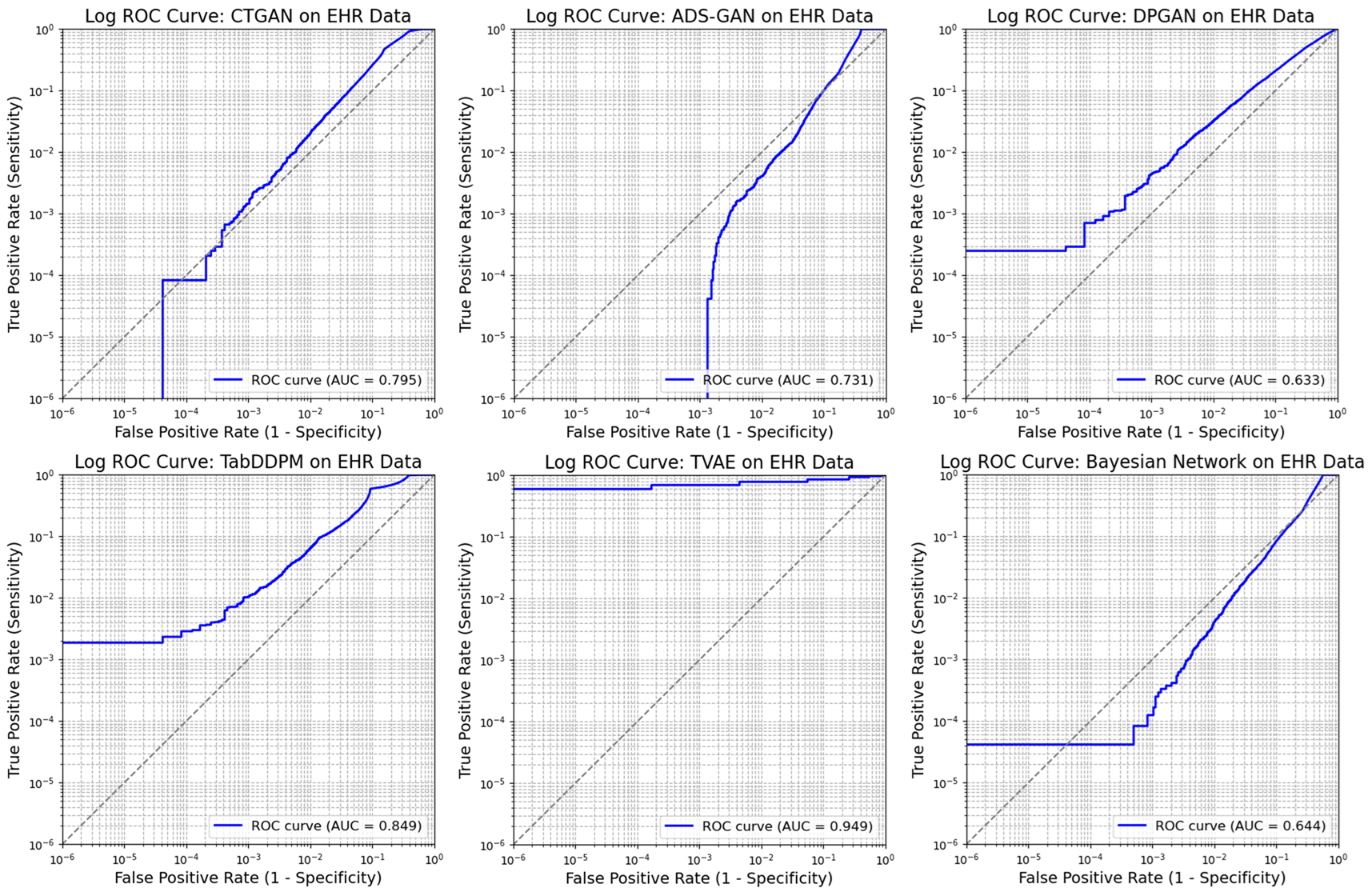}
    \caption{MIMIC-IV data: Log-ROC curves for the true distribution attack.}
    \label{fig:EHR_UB_RawDist_ROC}
    \end{figure}

\begin{figure}[h!] 
    \centering
    \includegraphics[width=\linewidth]{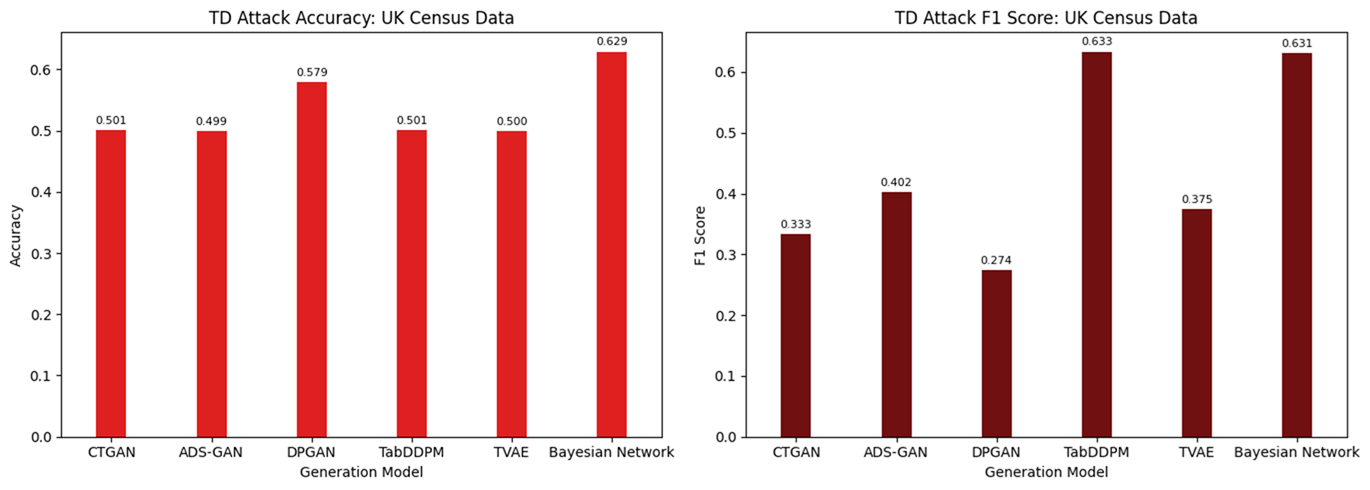}
    \caption{True distribution attack accuracies (left) and F1 scores (right) on the UK Census synthetic datasets.}
    \label{fig:UKCensus_UB_Risks}
    \end{figure}

\begin{figure}[h!]
   \centering
   \includegraphics[width=\linewidth]{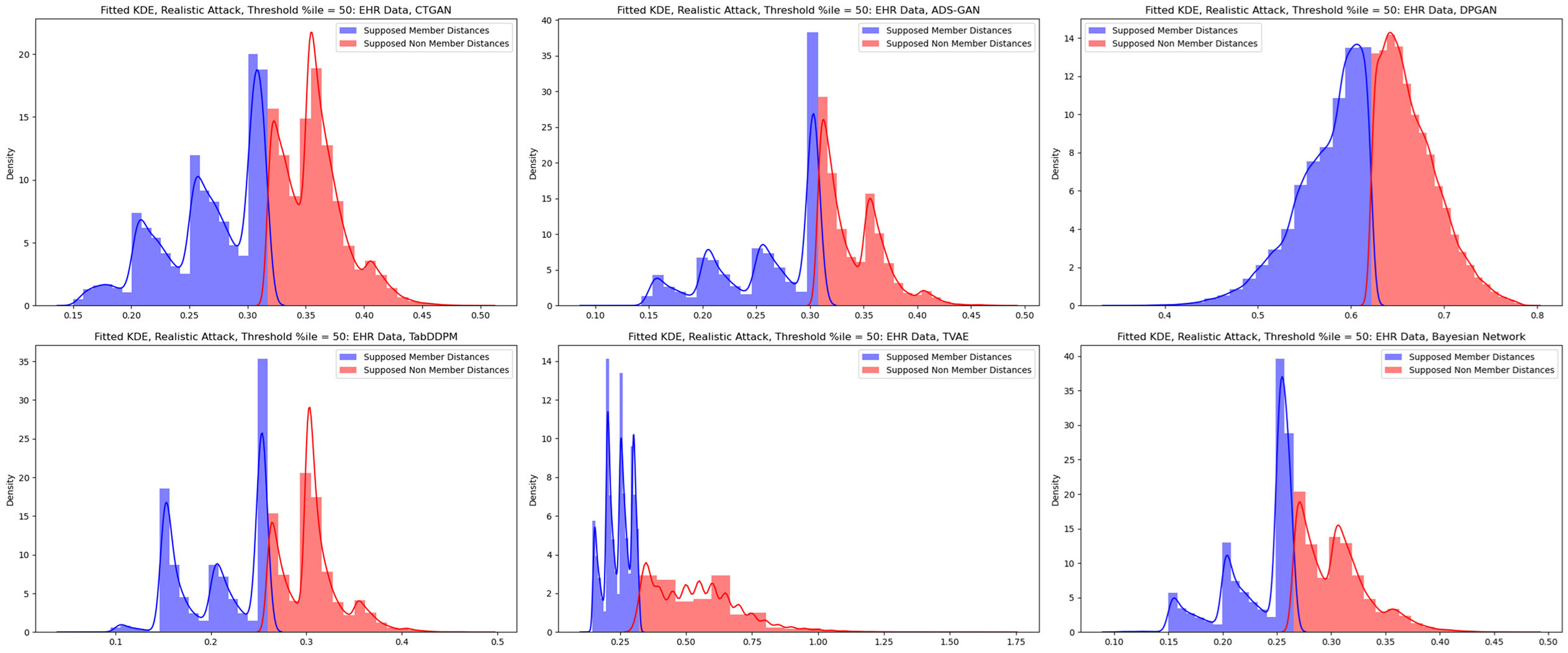}
   \caption{MIMIC-IV data: Distribution of distances for supposed members and supposed non-members determined by setting the distance threshold at 50 percentile. For each synthetic data, there were 56,000 supposed member distances and 56,000 supposed non-member distances in the training set.}
    \label{fig:EHR_Realistic_KDE}
    \end{figure}

\begin{figure}[h!]
    \centering
    \includegraphics[width=\linewidth]{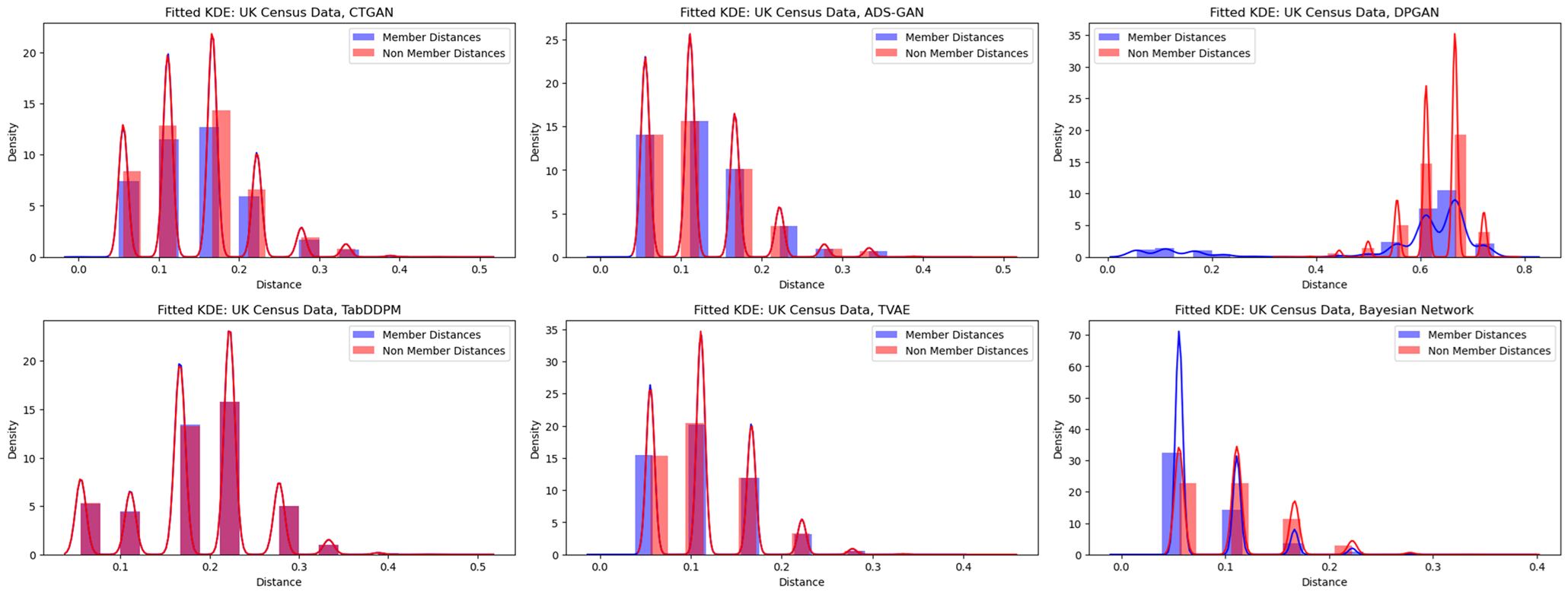}
    \caption{UK Census data: Distribution of member and non-member nearest-neighbor distances between the training records (in $D_{\text{attack}}$) and various synthetic datasets.}
    \label{fig:UKCensus_UB_KDE}
    \end{figure}
    
    \begin{figure}[h!]
    \centering
    \includegraphics[width=\linewidth]{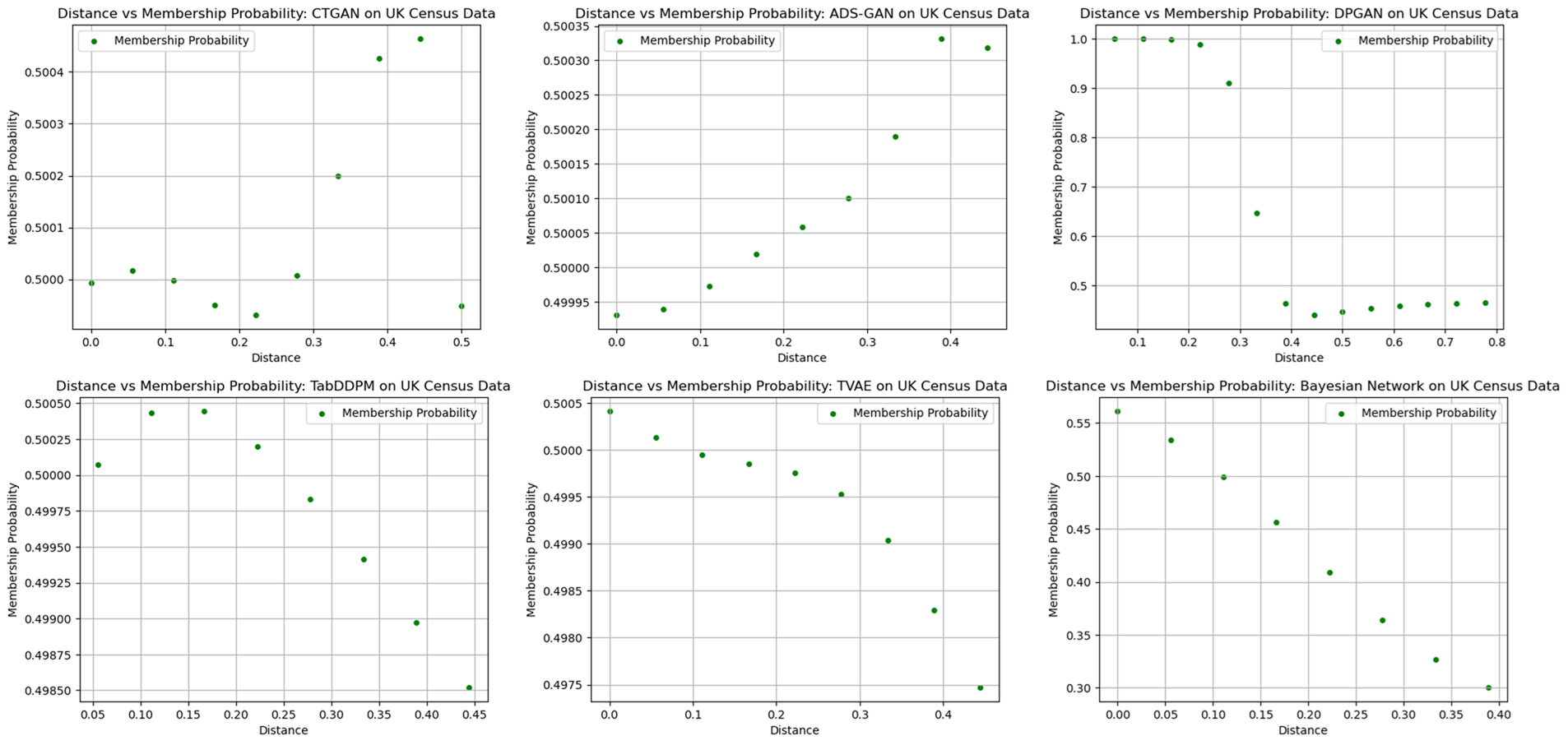}
    \caption{UK Census data: distance vs. membership probability for the test records.}
    \label{fig:UKCensus_UB_DistvsProb}
    \end{figure}

\begin{figure}[h!]
    \centering
    \includegraphics[width=\linewidth]{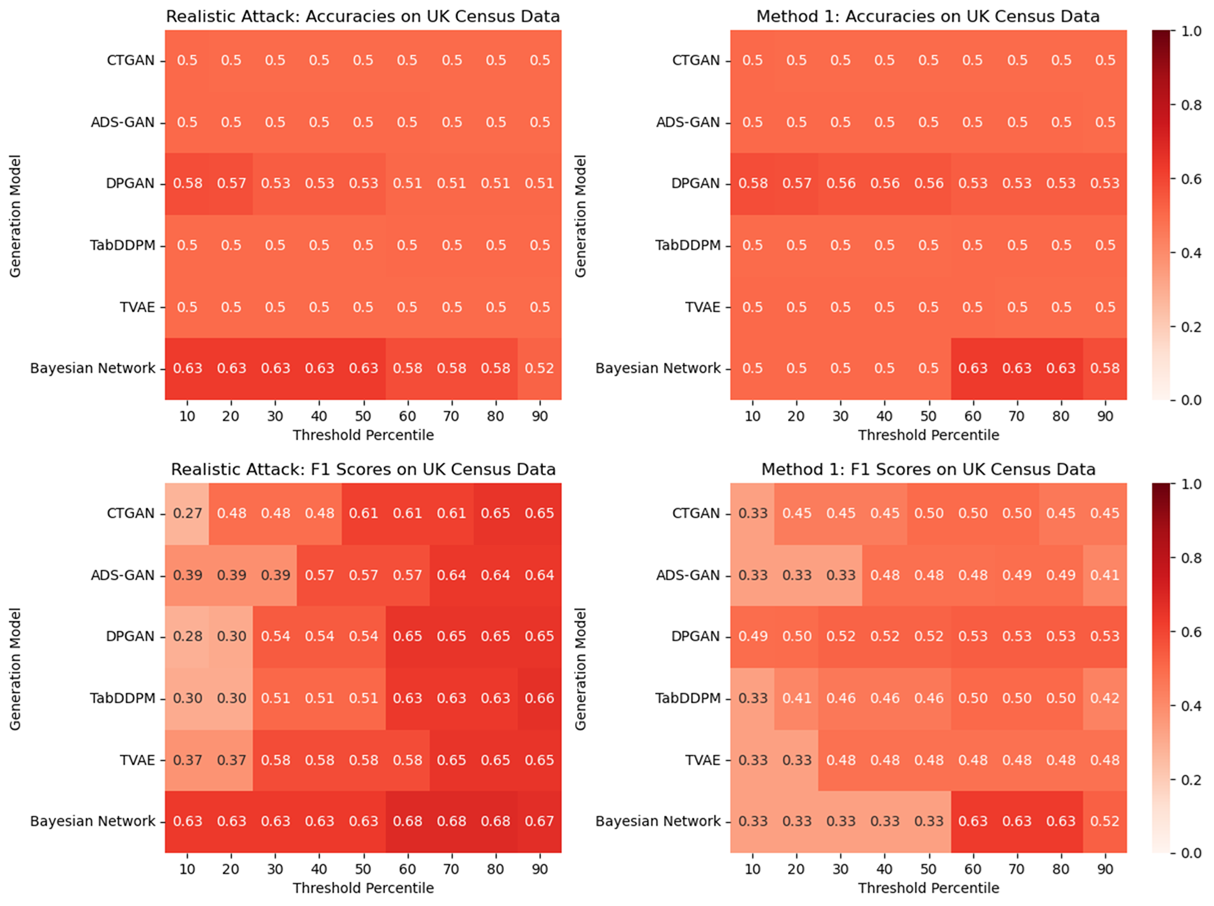}
    \caption{UK Census data: Accuracies (top) and F1 scores (bottom) for realistic attack vs. Method 1, across various distance thresholds.}
    \label{fig:UKCensus_Realistic_Risks}
    \end{figure}

\begin{figure}[h!]
   \centering
   \includegraphics[width=\linewidth]{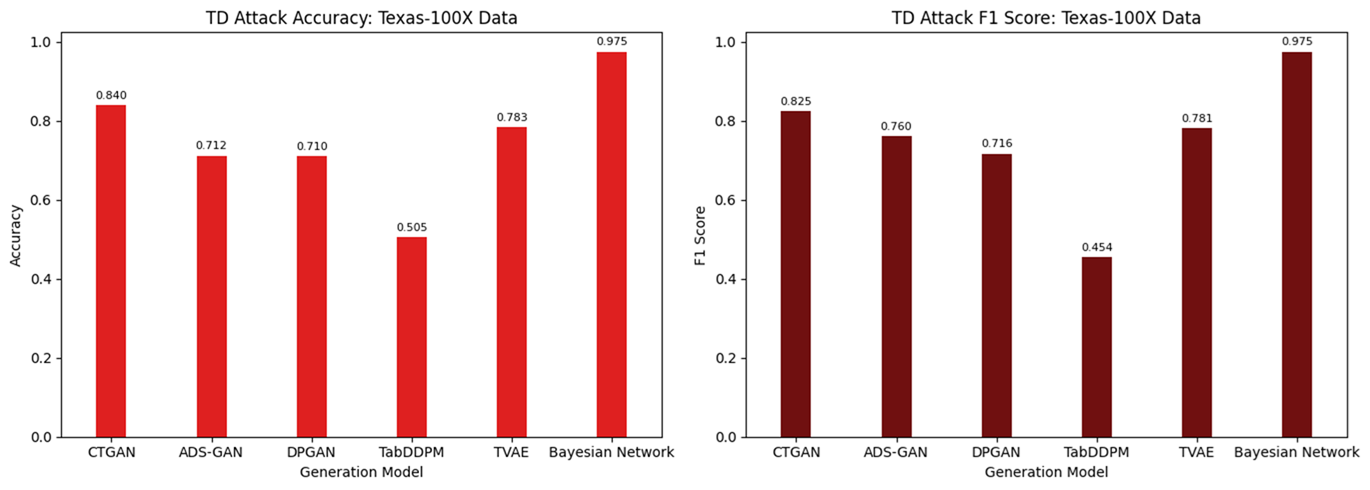}
   \caption{True distribution attack accuracies (left) and F1 scores (right) on the Texas-100X synthetic datasets.}
    \label{fig:Texas_UB_Risks}
    \end{figure}

\begin{figure}[h!]
    \centering
    \includegraphics[width=\linewidth]{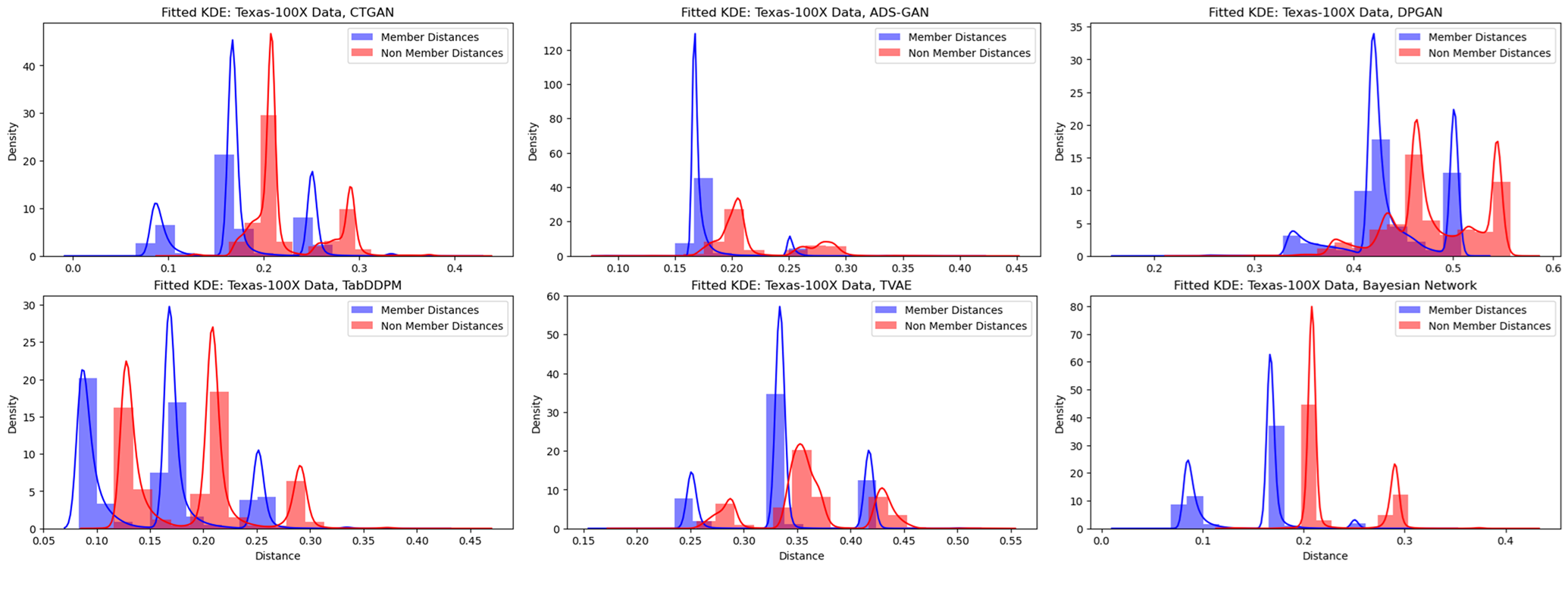}
    \caption{Texas-100X data: Distribution of nearest-neighbor distances between the training records (in $D_{\text{attack}}$) and various synthetic datasets. The member and non-member distances are modeled separately for each synthetic data.}
    \label{fig:Texas_UB_KDE}
    \end{figure}

\begin{figure}[h!]
    \centering
    \includegraphics[width=\linewidth]{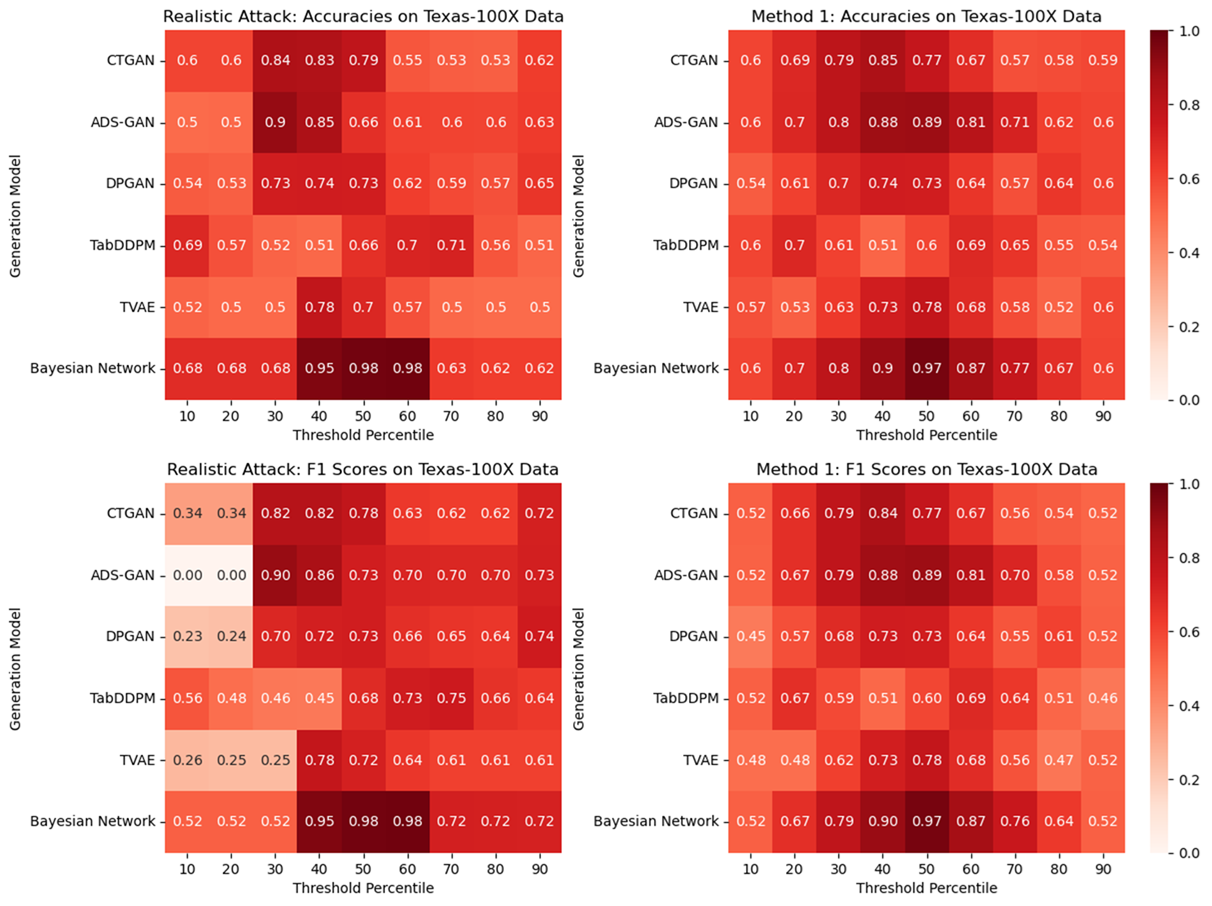}
    \caption{Texas-100X data: Accuracies (top) and F1 scores (bottom) for realistic attack vs. Method 1, across various distance thresholds.}
    \label{fig:Texas_Realistic_Risks}
    \end{figure}

\begin{figure}[h!]
   \centering
   \includegraphics[width=\linewidth]{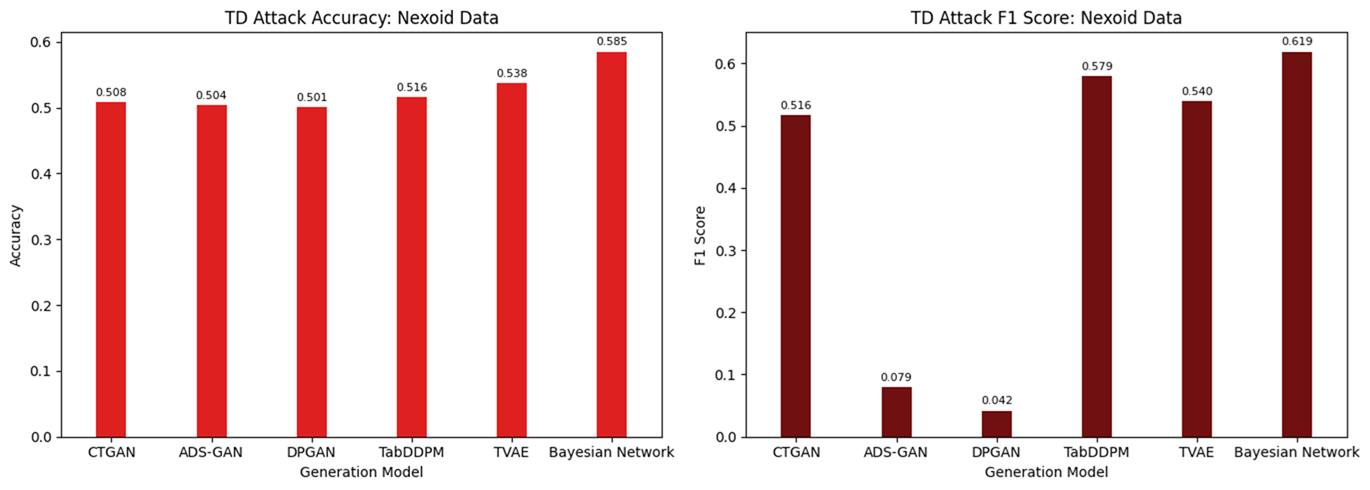}
   \caption{True distribution attack accuracies (left) and F1 scores (right) on the Nexoid synthetic datasets.}
    \label{fig:Nexoid_UB_Risks}
    \end{figure}

\begin{figure}[h!]
    \centering
    \includegraphics[width=\linewidth]{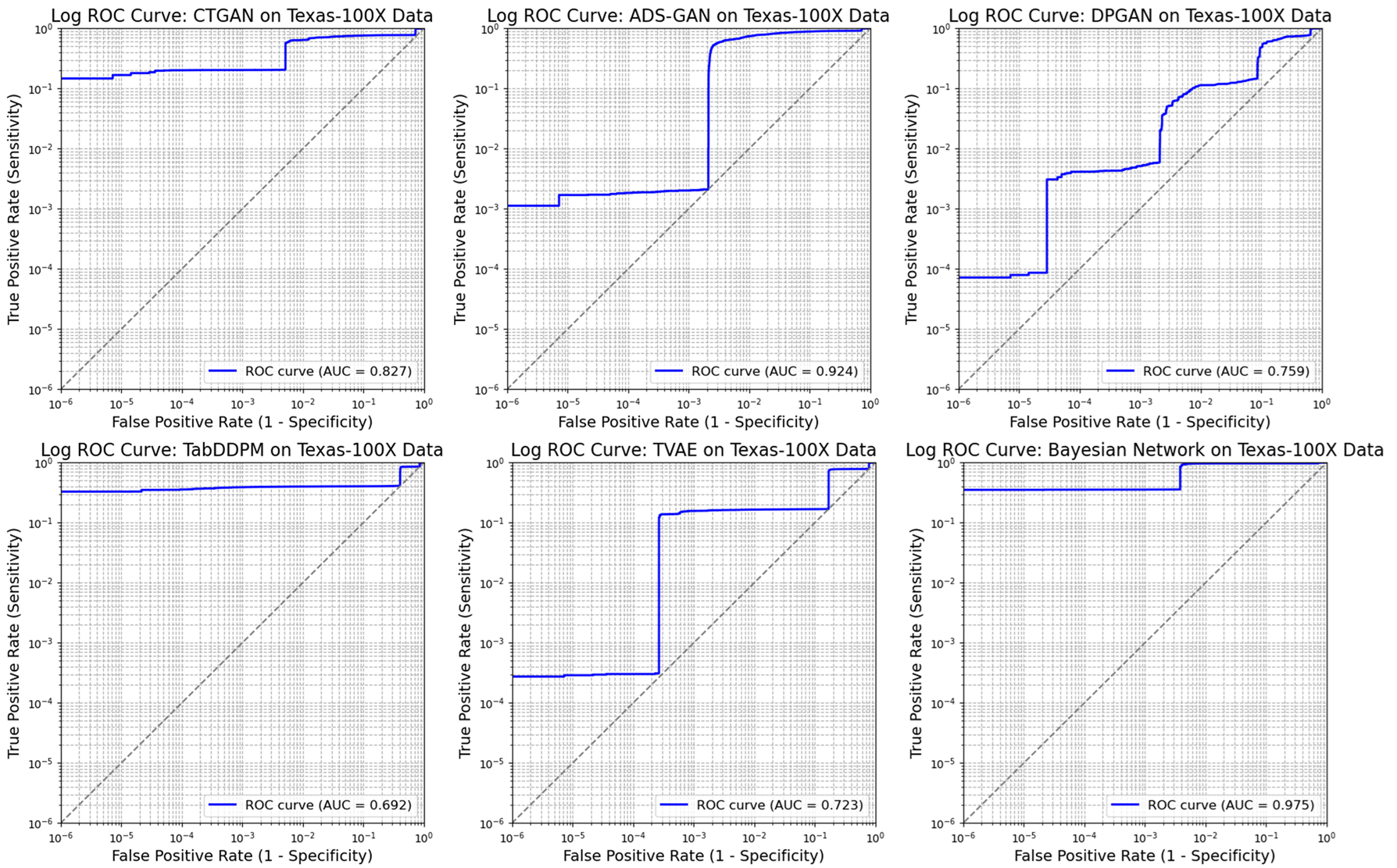}
    \caption{Texas-100X data: Log-ROC curves for the true distribution attack.}
    \label{fig:Texas_UB_RawDist_ROC}
    \end{figure}

\begin{figure}[h!]
    \centering
    \includegraphics[width=\linewidth]{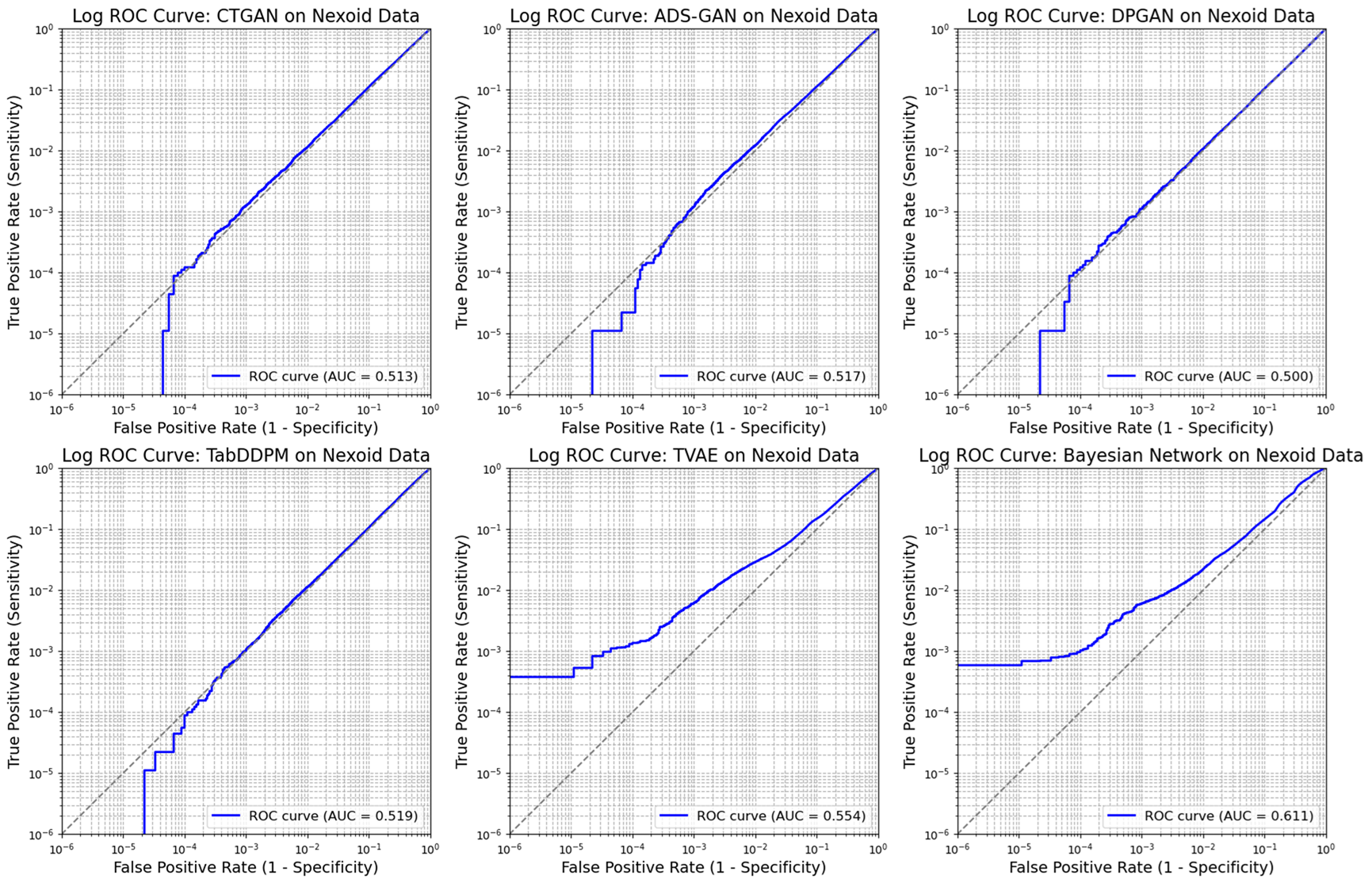}
    \caption{Nexoid data: Log-ROC curves for the true distribution attack.}
    \label{fig:Nexoid_UB_RawDist_ROC}
    \end{figure} 

\begin{figure}[h!]
   \centering
   \includegraphics[width=\linewidth]{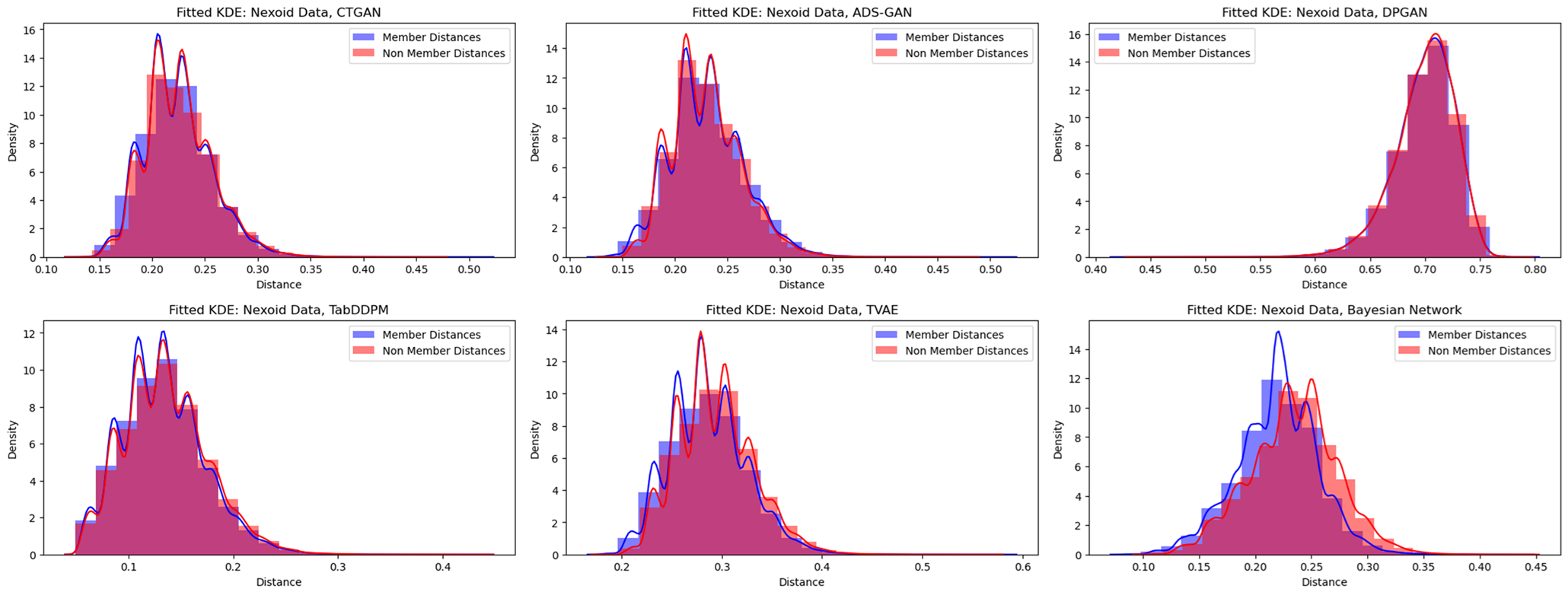}
   \caption{Nexoid data: Distribution of nearest neighbor distances between the training records (in $D_{\text{attack}}$) and various synthetic datasets.}
    \label{fig:Nexoid_UB_KDE}
    \end{figure}
    
    \begin{figure}[h!]
    \centering
    \includegraphics[width=\linewidth]{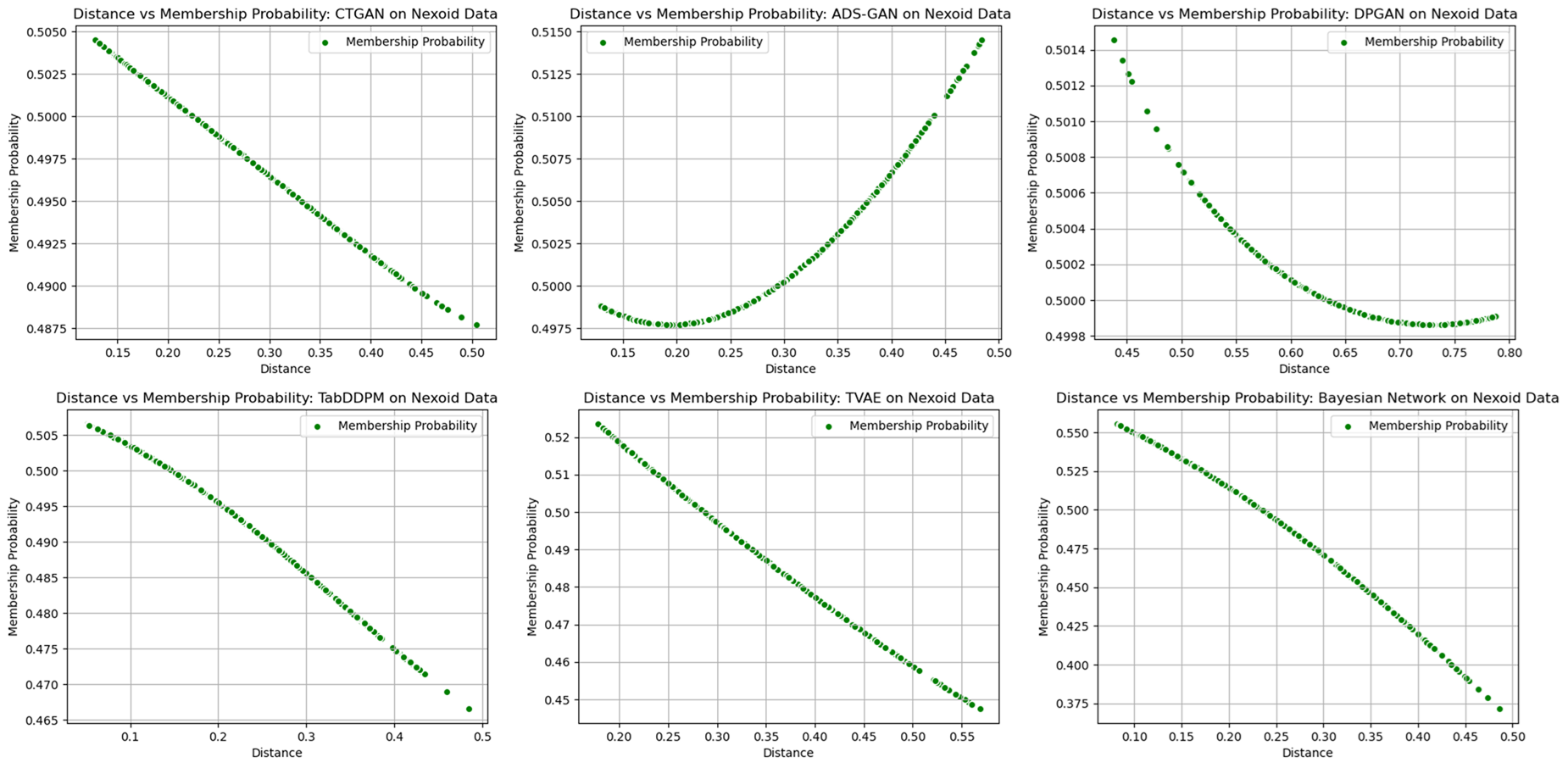}
    \caption{Nexoid data: distance vs. membership probability for the test records.}
    \label{fig:Nexoid_UB_DistvsProb}
    \end{figure}

\begin{figure}[h!]
   \centering
   \includegraphics[width=\linewidth]{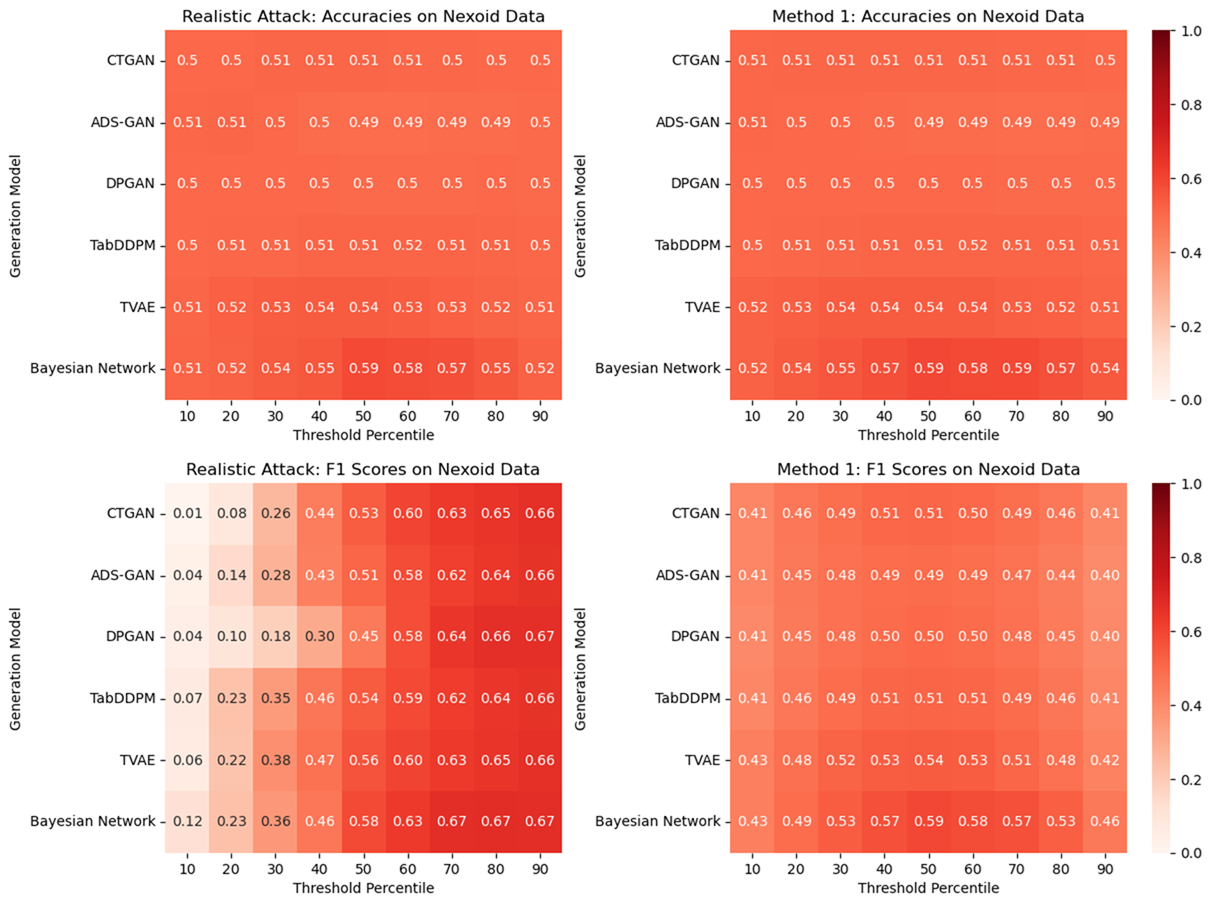}
   \caption{Nexoid data: Accuracies (top) and F1 scores (bottom) for realistic attack vs. Method 1, across various distance thresholds.}
    \label{fig:Nexoid_Realistic_Risks}
    \end{figure}

\end{document}